\documentclass[runningheads]{llncs}

\usepackage{times,tikz}
\usetikzlibrary{matrix,arrows,decorations.pathmorphing,calc,shadings,decorations.pathreplacing,patterns}
\usepackage{bbm}
\usepackage{amssymb,enumerate,multirow,commath,amsmath,graphicx}
\usepackage{amsfonts,bbding}
\usepackage{latexsym}

\def\natnum{{\mathbb N}}

\def\email#1{Email: {\tt #1}}

\spnewtheorem{thm}{Theorem}{\bf }{\it }
\spnewtheorem{prop}[thm]{Proposition}{\bf }{\it }
\spnewtheorem{prob}[thm]{Open Problem}{\bf }{\it }
\spnewtheorem{cor}[thm]{Corollary}{\bf }{\it }
\spnewtheorem{lem}[thm]{Lemma}{\bf }{\it }
\spnewtheorem{defn}[thm]{Definition}{\bf }{\rm }
\spnewtheorem{rem}[thm]{Remark}{\bf }{\rm }
\spnewtheorem{exmp}[thm]{Example}{\bf }{\rm }
\spnewtheorem{clm}[thm]{Claim}{\bf }{\it }
\spnewtheorem{nota}[thm]{Notation}{\bf }{\rm }
\spnewtheorem{ppt}[thm]{Property}{\bf }{\rm }
\spnewtheorem{fact}[thm]{Fact}{\bf }{\rm }

\def\nats{{\mathbb N}}
\def\rats{{\mathbb Q}}

\newcommand{\PBTD}{\mathrm{PBTD}}

\newcommand{\TD}{\mathrm{TD}}
\newcommand{\cl}{\mathrm{cl}}
\newcommand{\trcl}{\mathrm{trcl}}
\newcommand{\RTD}{\mathrm{RTD}}
\newcommand{\ord}{\mathrm{ord}}
\newcommand{\cL}{{\mathcal{L}}}
\newcommand{\Ra}{{\Rightarrow}}
\newcommand{\La}{{\Leftarrow}}
\newcommand{\cD}{{\mathcal{D}}}
\newcommand{\cF}{{\mathcal{F}}}
\newcommand{\cX}{{\mathcal{X}}}
\newcommand{\cS}{{\mathcal{S}}}
\newcommand{\cC}{{\mathcal{C}}}

\newcommand{\cR}{{\mathcal{R}}}
\newcommand{\cH}{{\mathcal{H}}}

\newcommand{\reals}{\mathbbm{R}}

\newcommand{\sm}{\setminus}
\newcommand{\seq}{\subseteq}
\newcommand{\ra}{\rightarrow}
\newcommand{\da}{\downarrow}
\newcommand{\ua}{\uparrow}
\newcommand{\LINSET}{\mathrm{LINSET}}
\newcommand{\CFLINSET}{\mathrm{CF}\mbox{-}\mathrm{LINSET}}

\newcommand{\dund}{\Longleftrightarrow}
\newcommand{\impl}{\Rightarrow}

\newcommand{\eset}{\emptyset}

\newcommand{\NELINSET}{\mathrm{NE}\mbox{-}\mathrm{LINSET}}
\newcommand{\NECFLINSET}{\mathrm{NE}\mbox{-}\mathrm{CF}\mbox{-}\mathrm{LINSET}}

\def\natnum{{\mathbb N}}

\newcommand{\bR}{{\mathbb{R}}}

\newcommand{\Sg}{\mathrm{Sg}}
\newcommand{\sumg}{\mathrm{sum}}
\newcommand{\tupleg}{\mathrm{tuple}}
\newcommand{\sign}{\mathrm{sign}}
\newcommand{\Ind}{\mathbbm{1}}
\newcommand\spn[1]{{\left\langle#1\right\rangle}}

\begin{document}

\title{Preference-based Teaching}

\titlerunning{Preference-based Teaching}

\author{Ziyuan Gao$^1$ \and Christoph Ries$^2$ \and Hans Ulrich Simon$^2$ \and Sandra Zilles$^1$}
\authorrunning{Z.~Gao, C.~Ries, H.~Simon and S.~Zilles}

\institute{Department of Computer Science\\University of Regina, Regina, SK, Canada S4S 0A2 
\\\email{\{gao257,zilles\}@cs.uregina.ca} \and
Department of Mathematics, \\
Ruhr-Universit\"{a}t Bochum, D-44780 Bochum, Germany\\\email{\{christoph.ries,hans.simon\}@rub.de}}

\maketitle

\begin{abstract}
We introduce a new model of teaching named 
``preference-based \linebreak[4]teaching'' and a corresponding complexity
parameter---the preference-based \linebreak[4]teaching dimension
(PBTD)---representing the worst-case number of examples 
needed to teach any concept in a given concept class. Although 
the PBTD coincides with the well-known recursive teaching
dimension (RTD) on finite classes, 
it is radically different on infinite ones: 
the RTD becomes infinite already for 
trivial infinite classes (such as half-intervals)
whereas the PBTD evaluates to 
reasonably small values for a wide collection
of infinite classes including classes consisting
of so-called closed sets w.r.t.~a given closure 
operator, including various classes related to
linear sets over $\nats_0$ (whose RTD had been 
studied quite recently) and including the class 
of Euclidean half-spaces. On top of presenting these concrete 
results, we provide the reader with a theoretical 
framework (of a combinatorial flavor) which 
helps to derive bounds on the PBTD.
\end{abstract}


\section{Introduction} \label{sec:introduction}

The classical model of teaching~\cite{SM1991,GK1995}
formulates the following interaction protocol between 
a teacher and a student:
\begin{itemize}
\item
Both of them agree on a ``classification-rule system'', 
formally given by a concept class $\cL$.
\item
In order to teach a specific concept $L \in \cL$, 
the teacher presents to the student a \emph{teaching set}, 
i.e., a set $T$ of labeled examples so that $L$ is the only 
concept in $\cL$ that is consistent with $T$.
\item
The student determines $L$ as the unique concept in $\cL$ 
that is consistent with $T$.
\end{itemize}

Goldman and Mathias~\cite{GM1996} pointed out that this model of teaching 
is not powerful enough, since the teacher is required 
to make \emph{any\/} consistent learner successful. 
A challenge is to model powerful teacher/student 
interactions without enabling unfair ``coding tricks''. Intuitively, the term ``coding trick'' 
refers to any form of undesirable collusion between teacher and learner, which would 
reduce the learning process to a mere decoding of a code the teacher sent to the learner.
There is no generally accepted definition of what constitutes a coding trick, in part because 
teaching an exact learner could always be considered coding to some extent: the teacher presents a 
set of examples which the learner ``decodes'' into a concept. 

In this paper, we adopt the notion of 
``valid teacher/learner pair'' introduced by \cite{GM1996}. They consider their model to be intuitively 
free of coding tricks while it provably allows for a much broader 
class of interaction protocols than the original teaching 
model. In particular, teaching may thus become more efficient 
in terms of the number of examples in the teaching sets. 
Further definitions of how to avoid unfair coding tricks 
have been suggested~\cite{ZLHZ2011}, but they were less 
stringent than the one proposed by Goldman and Mathias. The latter simply requests that, if the learner 
hypothesizes concept $L$ upon seeing a sample set $S$ of labeled examples, then the learner will still 
hypothesize $L$ when presented with any sample set $S\cup S'$, where $S'$ contains only examples labeled 
consistently with $L$. A coding trick would then be any form of exchange between the teacher and the learner that 
does not satisfy this definition of validity.

The model of recursive teaching~\cite{ZLHZ2011,MGZ2014}, 
which is free of coding tricks according to the Goldman-Mathias 
definition, has recently gained attention because its complexity 
parameter, the recursive teaching dimension (RTD), has shown 
relations to the VC-dimension and to sample 
compression~\cite{ChenCT16,DFSZ2014,MSWY2015,SZ2015}, 
when focusing on finite concept classes. Below though we will 
give examples of rather simple infinite concept classes with 
infinite RTD, suggesting that the RTD is inadequate for 
addressing the complexity of teaching infinite classes.

In this paper, we introduce a model 
called \emph{preference-based teaching}, in which the teacher 
and the student do not only agree on a classification-rule 
system $\cL$ but also on a preference relation (a strict partial 
order) imposed on $\cL$. If the labeled examples presented by the
teacher allow for several consistent explanations (= consistent
concepts) in $\cL$, the student will choose a concept $L \in \cL$
that she prefers most. This gives more flexibility to the teacher 
than the classical model: the set of labeled examples need not 
distinguish a target concept $L$ from any other concept in $\cL$ 
but only from those concepts $L'$ over which $L$ is not preferred.\footnote{Such a preference relation can be thought of as a kind of bias in learning: the student is ``biased'' towards concepts that are preferred over others, and the teacher, knowing the student's bias, selects teaching sets accordingly.}
At the same time, preference-based teaching yields valid 
teacher/learner pairs according to Goldman and Mathias's definition.
We will show that the new model, despite avoiding coding tricks, 
is quite powerful. Moreover, as we will see in the course of the 
paper, it often allows for a very natural design of teaching sets.

Assume teacher and student choose a preference relation that 
minimizes the worst-case number $M$ of examples required for 
teaching any concept in the class $\mathcal{L}$. This number $M$ 
is then called the preference-based teaching dimension (PBTD) 
of $\mathcal{L}$. In particular, we will show the following:

(i) Recursive teaching is a special case of preference-based 
teaching where the preference relation satisfies a so-called 
``finite-depth condition''. It is precisely this additional 
condition that renders recursive teaching useless for many 
natural and apparently simple infinite concept classes. 
Preference-based teaching successfully addresses these 
shortcomings of recursive teaching, see Section~\ref{sec:rtd}. 
For finite classes, PBTD and RTD are equal.

(ii) A wide collection of geometric and algebraic concept 
classes with infinite RTD can be taught very efficiently, 
i.e., with low PBTD. To establish such results, we show 
in Section~\ref{sec:closure-operator} that spanning sets 
can be used as preference-based teaching sets with positive 
examples only --- a result that is very simple to obtain but
quite useful.

(iii) In the preference-based model, linear sets over 
$\nats_0$ with origin 0 and at most $k$ generators can be taught with 
$k$ positive examples, while recursive teaching with a bounded number 
of positive examples was previously shown to be impossible and it is 
unknown whether recursive teaching with a bounded number of positive 
and negative examples is possible for $k \geq 4$. We also give 
some almost matching 
upper and lower bounds on the PBTD for other classes of linear sets, 
see Section~\ref{sec:linsets}.

(iv) The PBTD of halfspaces in $\reals^d$ is upper-bounded by $6$, independent of the dimensionality $d$ 
(see Section~\ref{sec:halfspaces}), 
while its RTD is infinite. 

(v) We give full characterizations of concept classes that can be taught with only one example (or with only one example, which is positive) in the preference-based model (see Section~\ref{sec:pbtd1}).

Based on our 
results and the naturalness of the teaching sets and preference 
relations used in their proofs, we claim that 
preference-based teaching is far more 
suitable to the study of infinite concept classes than recursive 
teaching. 

Parts of this paper were published in a previous conference version~\cite{GRSZ2016}.

\section{Basic Definitions and Facts} \label{sec:definitions}

$\nats_0$ denotes the set of all non-negative integers and
$\nats$ denotes the set of all positive integers. 
A {\em concept class} $\cL$ is a family of subsets over a universe $\cX$, 
i.e., $\cL \seq 2^\cX$ where $2^\cX$ denotes the powerset of $\cX$.
The elements of $\cL$ are called {\em concepts}. A {\em labeled example} 
is an element of $\cX \times\{-,+\}$. We slightly deviate from this notation in 
Section~\ref{sec:halfspaces}, where our treatment of halfspaces makes it 
more convenient to use $\{-1,1\}$ instead of $\{-,+\}$, and in Section~\ref{sec:pbtd1}, 
where we perform Boolean operations on the labels and therefore use 
$\{0,1\}$ instead of $\{-,+\}$. Elements of $\cX$ are called 
{\em examples}. Suppose that $T$ is a set of labeled examples. 
Let $T^+ = \{x\in\cX : (x,+) \in T\}$ and $T^- = \{x\in\cX : (x,-) \in T\}$.
A set $L \seq\cX$ is {\em consistent with $T$} if it
includes all examples in $T$ that are labeled ``$+$'' and excludes
all examples in $T$ that are labeled ``$-$'', i.e, if
$T^+ \seq L$ and $T^- \cap L = \eset$. A set of labeled examples that
is consistent with $L$ but not with $L'$ is said to {\em distinguish $L$
from $L'$}. The classical model of teaching is then defined as follows.

\begin{definition}[\cite{SM1991,GK1995}] \label{def:classical-model}
A {\em teaching set} for a concept $L \in \cL$ w.r.t.~$\cL$ is a set $T$ of labeled
examples such that $L$ is the only concept in $\cL$ that is consistent
with $T$, i.e., $T$ distinguishes $L$ from any other concept in $\cL$. Define 
$\TD(L,\cL) = \inf\{|T| : T\mbox{ is a teaching\ }$ $\mbox{set for $L$ w.r.t.~$\cL$}\}$.
i.e., $\TD(L,\cL)$ is the smallest possible size of a teaching set for $L$ 
w.r.t.~$\cL$.  If $L$ has no finite teaching set w.r.t.~$\cL$, 
then $\TD(L,\cL)=\infty$. 
The number $\TD(\cL) = \sup_{L \in \cL}\TD(L,\cL) \in \nats_0\cup\{\infty\}$
is called the {\em teaching dimension of $\cL$}.
\end{definition}

For technical reasons, we will occasionally deal with the 
number 
$\TD_{min}(\cL) = \inf_{L \in \cL}\TD(L,$ $\cL)$, i.e., the number of
examples needed to teach the concept from $\cL$ that is easiest to
teach. 

In this paper, we will examine a teaching model in which the teacher
and the student do not only agree on a classification-rule system $\cL$
but also on a preference relation, denoted as $\prec$, imposed on $\cL$.
We assume that $\prec$ is a {\em strict partial order} on $\cL$,
i.e., $\prec$ is asymmetric and transitive. 
The partial order that makes every pair $L \neq L' \in \cL$ 
incomparable 
is denoted by $\prec_\eset$. For every $L \in\cL$, let 
\[ \cL_{\prec L} = \{L'\in\cL: L' \prec L\} \]
be the set of concepts over which $L$ is strictly preferred. Note 
that $\cL_{\prec_\eset L} = \eset$ for every $L\in\cL$.

As already noted above, a teaching set $T$ of $L$ w.r.t.~$\cL$ 
distinguishes $L$ from any other concept in $\cL$. If a preference 
relation comes into play, then $T$ will be exempted from the 
obligation to distinguish $L$ from the concepts in $\cL_{\prec L}$ 
because $L$ is strictly preferred over them anyway. 

\begin{definition} \label{def:td-prec-L}
A {\em teaching set for $L \seq X$ w.r.t.~$(\cL,\prec)$} is defined as 
a teaching set for $L$ w.r.t.~$\cL\sm\cL_{\prec L}$. Furthermore define 
\[ 
\PBTD(L,\cL,\prec) = \inf\{|T| : T\mbox{ is a teaching set for $L$ w.r.t.~$(\cL,\prec$})\}
\in \nats_0\cup\{\infty\} \enspace .
\]
The number $\PBTD(\cL,\prec) = \sup_{L \in \cL}\PBTD(L,\cL,\prec) \in \nats_0\cup\{\infty\}$ 
is called the {\em teaching dimension of $(\cL,\prec)$}. 
\end{definition}
Definition~\ref{def:td-prec-L} implies that
\begin{equation} \label{eq:td-prec-L}
\PBTD(L,\cL,\prec) = \TD(L,\cL\sm\cL_{\prec L}) \enspace . 
\end{equation}
Let $L \mapsto T(L)$ be a mapping that assigns a teaching 
set for $L$ w.r.t.~$(\cL,\prec)$ to every $L\in\cL$. It
is obvious from Definition~\ref{def:td-prec-L} that $T$
must be injective, i.e., $T(L) \neq T(L')$ if $L$ and $L$'
are distinct concepts from $\cL$.
The classical model of teaching is obtained from the model
described in Definition~\ref{def:td-prec-L} when we plug in the
empty preference relation $\prec_\eset$ for $\prec$. 
In particular, $\PBTD(\cL,\prec_\eset)$ $= \TD(\cL)$.

We are interested in finding the partial order 
that is optimal for the purpose of teaching and we aim at 
determining the corresponding teaching dimension.
This motivates the following notion:
\begin{definition} \label{def:td-best-le}
The {\em preference-based teaching dimension of $\cL$} is given by
\[
\PBTD(\cL) = \inf\{\PBTD(\cL,\prec) : \mbox{$\prec$ is a strict partial order on $\cL$}\}
\enspace .
\]
\end{definition}

A relation $R'$ on $\cL$ is said to be an {\em extension of a relation $R$}
if $R \seq R'$. The {\em order-extension principle} states that any partial order 
has a linear extension \cite{Jech-1973}. The following result (whose second assertion follows 
from the first one in combination with the order-extension principle) is 
pretty obvious:

\begin{lemma} \label{lem:extension}
\begin{enumerate}
\item
Suppose that $\prec'$ extends $\prec$. If $T$ is a teaching set for $L$
w.r.t.~$(\cL,\prec)$, then $T$ is a teaching set for $L$ w.r.t.~$(\cL,\prec')$.
Moreover $\PBTD(\cL,\prec') \le \PBTD(\cL,$ $\prec)$.
\item
$\PBTD(\cL) = \inf\{\PBTD(\cL,\prec) : \mbox{$\prec$ is a strict linear order on $\cL$}\}$.
\end{enumerate}
\end{lemma}

Recall that Goldman and Mathias \cite{GM1996} suggested to avoid coding tricks by requesting 
that any superset $S$ of a teaching set for a concept $L$ remains a teaching set, 
if $S$ is consistent with $L$. This property is obviously satisfied in 
preference-based teaching. A preference-based teaching set needs to distinguish 
a concept $L$ from all concepts in $\cL$ that are preferred over $L$. 
Adding more labeled examples from $L$ to such a teaching set will still 
result in a set distinguishing $L$ from all concepts in $\cL$ that are 
preferred over $L$.

\paragraph{Preference-based teaching with positive examples only.}

Suppose that $\cL$ contains two concepts $L,L'$ such that $L \subset L'$.
In the classical teaching model, any teaching set for $L$ w.r.t.~$\cL$ has to employ 
a negative example in order to distinguish $L$ from $L'$. Symmetrically, any
teaching set for $L'$ w.r.t.~$\cL$ has to employ a positive example. Thus 
classical teaching cannot be performed with one
type of examples only unless $\cL$ is an antichain w.r.t.~inclusion. As for
preference-based teaching, the restriction to one type of examples is much less
severe, as our results below will show. 

A teaching set $T$ for $L \in \cL$ w.r.t.~$(\cL,\prec)$ is said to be 
{\em positive} if it does not make use of negatively labeled examples,
i.e., if $T^- = \eset$. In the sequel, we will occasionally identify
a positive teaching set $T$ with $T^+$. A positive teaching set 
for $L$ w.r.t.~$(\cL,\prec)$ can clearly not distinguish $L$ from a 
proper superset of $L$ in $\cL$. Thus, the following holds:

\begin{lemma} \label{lem:ts-pos}
Suppose that $L \mapsto T^+(L)$ maps each $L \in \cL$ to a
positive teaching set for $L$ w.r.t.~$(\cL,\prec)$. 
Then $\prec$ must be an extension of $\supset$ (so that proper
subsets of a set $L$ are strictly preferred over $L$) and, 
for every $L \in \cL$, the set $T^+(L)$ must distinguish $L$ 
from every proper subset of $L$ in $\cL$. 
\end{lemma}

\noindent
Define
\begin{equation} \label{eq:td-plus-L}
\PBTD^+(L,\cL,\prec) = \inf\{|T| : T\mbox{ is a positive teaching set for $L$ w.r.t.~$(\cL,\prec$})\}
\enspace .
\end{equation}
The number $\PBTD^+(\cL,\prec) = \sup_{L \in \cL}\PBTD^+(L,\cL,\prec)$ (possibly $\infty$) 
is called the {\em positive teaching dimension of $(\cL,\prec)$}. 
The {\em positive preference-based teaching dimension of $\cL$} is then 
given by
\begin{equation} \label{eq:td-plus-cL}
\PBTD^+(\cL) = \inf\{\PBTD^+(\cL,\prec) : \mbox{$\prec$ is a strict partial order on $\cL$}\}
\enspace .
\end{equation}

\paragraph{Monotonicity.}

A complexity measure $K$ that assigns a number $K(\cL) \in \nats_0$ 
to a concept class $\cL$ is said to be {\em monotonic} if $\cL' \seq \cL$
implies that $K(\cL') \le K(\cL)$. It is well known (and trivial to see)
that $\TD$ is monotonic. It is fairly obvious that $\PBTD$ is
monotonic, too:

\begin{lemma} \label{lem:monotonicity}
$\PBTD$ and $\PBTD^+$ are monotonic.
\end{lemma}


\noindent
As an application of monotonicity, we show the following result:

\begin{lemma} \label{lem:lb-tdmin}
For every finite subclass $\cL'$ of $\cL$, we have
$ \PBTD(\cL) \ge \PBTD(\cL') \ge \TD_{min}(\cL') $.
\end{lemma}

\begin{proof}
The first inequality holds because $\PBTD$ is monotonic. The second 
inequality follows from the fact that a finite partially ordered set
must contain a minimal element. Thus, for any fixed choice of $\prec$,
$\cL'$ must contain a concept $L'$ such that $\cL'_{\prec L'} = \eset$.
Hence, 
\[
\PBTD(\cL',\prec) \ge \PBTD(L',\cL',\prec)  \stackrel{(\ref{eq:td-prec-L})}{=}
\TD(L',\cL'\sm\cL'_{\prec L'}) = \TD(L',\cL') \ge \TD_{min}(\cL') \enspace .
\]
Since this holds for any choice of $\prec$, we 
get $\PBTD(\cL') \ge \TD_{min}(\cL')$, as desired.
\end{proof}

\section{Preference-based versus Recursive Teaching} \label{sec:rtd}

The preference-based teaching dimension is a relative of the
recursive teaching dimension. In fact, both notions coincide on 
finite classes, as we will see shortly. 
We first recall the definitions of the recursive 
teaching dimension and of some related notions~\cite{ZLHZ2011,MGZ2014}.

A {\em teaching sequence for $\cL$} is a sequence of the 
form $\cS = (\cL_i,d_i)_{i\ge1}$ where $\cL_1,\cL_2,\cL_3,\ldots$ 
form a partition of $\cL$ into non-empty sub-classes and, for every $i\ge1$, 
we have that 
\begin{equation} \label{eq:rtd}
d_i = \sup_{L\in\cL_i}\TD\left(L,\cL\sm\cup_{j=1}^{i-1}\cL_j\right) \enspace .
\end{equation}
If, for every $i\ge 1$, $d_i$ is the supremum over all $L \in \cL_i$ 
of the smallest size of a \emph{positive teaching set} for $L$ 
w.r.t.\ $\cup_{j\geq i}\cL_j$ (and $d_i = \infty$ if some $L \in \cL_i$ 
does not have a positive teaching set w.r.t.\ $\cup_{j\ge i}\cL_j$), 
then $\cS$ is said to be a \emph{positive teaching sequence for $\cL$}.  
The {\em order} of a teaching sequence or a positive teaching
sequence $\cS$ (possibly $\infty$) is defined as $\ord(\cS) = \sup_{i\ge1}d_i$. 
The {\em recursive teaching dimension of $\cL$} (possibly $\infty$) 
is defined as the order of the teaching sequence of lowest order for $\cL$. 
More formally, $\RTD(\cL) = \inf_{\cS}\ord(\cS)$ where $\cS$
ranges over all teaching sequences for $\cL$. 
Similarly, $\RTD^+(\cL) = \inf_{\cS}\ord(\cS)$, where
$\cS$ ranges over all positive teaching sequences for $\cL$.
Note that the following holds for every $\cL' \seq \cL$ and for every teaching 
sequence $\cS = (\cL_i,d_i)_{i\ge1}$ for $\cL'$ such that $\ord(\cS) = \RTD(\cL')$:
\begin{equation} \label{eq:rtd-tdmin}
\RTD(\cL) \ge \RTD(\cL') = \ord(\cS) \ge 
d_1 = \sup_{L\in\cL_1}\TD(L,\cL') \ge \TD_{min}(\cL') \enspace .
\end{equation}

Note an important difference between $\PBTD$ and $\RTD$: while $\RTD(\cL) \ge$ $\TD_{min}$ $(\cL')$ 
for \emph{all}\/ $\cL'\subseteq\cL$, in general the same holds for $\PBTD$ only when restricted to finite $\cL'$, 
cf.\ Lemma~\ref{lem:lb-tdmin}. This difference will become evident in the proof of Lemma~\ref{lem:huge-gap}.

The {\em depth} of $L \in \cL$ w.r.t.~a strict partial order 
imposed on $\cL$ is defined as the length of the longest 
chain in $(\cL,\prec)$ that ends with the $\prec$-maximal element $L$ (resp.~as $\infty$
if there is no bound on the length of these chains).
The recursive teaching dimension is related to the 
preference-based teaching dimension as follows:

\begin{lemma} \label{lem:rtd-pbtd}
$\RTD(\cL) = \inf_{\prec}\PBTD(\cL,\prec)$ 
and $\RTD^+(\cL) = \inf_{\prec}\PBTD^+(\cL,\prec)$ where $\prec$ ranges
over all strict partial orders on $\cL$ that satisfy the following
``finite-depth condition'': every $L \in \cL$ has a finite depth
w.r.t.~$\prec$.
\end{lemma}

The following is an immediate consequence of Lemma~\ref{lem:rtd-pbtd}
and the trivial observation that the finite-depth condition is always
satisfied if $\cL$ is finite:

\begin{corollary} \label{cor:rtd-pb}
$\PBTD(\cL) \le \RTD(\cL)$,
with equality if $\cL$ is finite.
\end{corollary}

\noindent
While $\PBTD(\cL)$ and $\RTD(\cL)$ refer to the same finite number when $\cL$ is finite, there are classes for which $\RTD$ is finite and yet larger than $\PBTD$, as Lemma~\ref{lem:huge-gap} will show. Generally, for infinite classes, the gap between $\PBTD$ and $\RTD$
can be arbitrarily large:

\begin{lemma} \label{lem:huge-gap}
There exists an infinite class $\cL_\infty$ 
of VC-dimension $1$ such that $\PBTD^+$ $(\cL_\infty)$ $=1$
and $\RTD(\cL_\infty)=\infty$. Moreover, for every $k\ge1$,
there exists an infinite class $\cL_k$ such 
that $\PBTD^+(\cL_k)=1$ and $\RTD(\cL_k)=k$.
\end{lemma}

\begin{proof}
We first show that there exists a class of VC-dimension $1$,
say $\cL_\infty$, such that $\PBTD^+(\cL_\infty)=1$
while $\RTD(\cL_\infty)=\infty$. To this end,
let $\cL_\infty$ be the family of closed half-intervals over $[0,1)$, 
i.e., $\cL_\infty = \{ [0,a]: 0 \le a < 1\}$. We first prove 
that $\PBTD^+(\cL_\infty) = 1$. Consider the preference 
relation given by $[0,b] \prec [0,a]$ iff $a<b$. 
Then, for each $0 \le a <1$, we have 
\[
\PBTD([0,a],\cL_\infty,\prec) \stackrel{(\ref{eq:td-prec-L})}{=}
\TD([0,a],\{[0,b]:\ 0 \le b \le a\}) = 1
\] 
because the single example $(a,+)$ suffices for 
distinguishing $[0,a]$ from any interval $[0,b]$ 
with $b<a$. 

It was observed by~\cite{MSWY2015} already
that $\RTD(\cL_\infty)=\infty$ because every teaching
set for some $[0,a]$ must contain an infinite sequence 
of distinct reals that converges from above to $a$.
Thus, using Equation~(\ref{eq:rtd-tdmin}) 
with $\cL'=\cL$, we have $\RTD(\cL_\infty) \ge \TD_{min}(\cL_\infty) = \infty$.

Next we show that, for every $k\ge1$, there exists 
a class, say $\cL_k$, such that $\PBTD^+$ $(\cL_k)=1$ 
while $\RTD(\cL_k)=k$. To this end, let $\cX = [0,2)$.
For each $a \in [0,1)$, fix a binary representation
$\sum_{n \ge 1}\alpha_n2^{-n}$ of $a$,  
where $\alpha_n\in\{0,1\}$ are binary coefficients, and for 
all $i=1,\ldots,k$, let $1 \le a_i < 2$ be given 
by $a_i = 1+\sum_{n \ge 0}\alpha_{kn+i}2^{-kn+i}$.\footnote{Note that, for $a=\frac{m}{2^N}$ with $m,N\in\mathbb{N}$, there are two binary representations. We can pick either one to define the $\alpha_n$ and $a_i$ values.}
Let $A$ be the set of all $a \in [0,1)$ such that
if $\sum_{n \ge 1}\alpha_n2^{-n}$ is the binary representation
of $a$ fixed earlier, then for all $i \in \{1,\ldots,k\}$,
there is some $n \geq 0$ for which $\alpha_{nk+i} \neq 0$. Finally,
let $I_a = [0,a] \cup \{a_1,\ldots,a_k\} \seq \cX$
and let $\cL_k = \{I_a: 0 \le a < 1 \wedge a \in A\}$. 
Clearly $\PBTD^+(\cL_k)=1$ because, using the preference
relation given by $I_b \prec I_a$ iff $a<b$, we can 
teach $I_a$ w.r.t.~$\cL_k$ by presenting the single
example $(a,+)$ (the same strategy as for half-intervals).
Moreover, note that $I_a$ is the only concept in $\cL_k$
that contains $a_1,\ldots,a_k$, i.e., $\{a_1,\ldots,a_k\}$
is a positive teaching set for $I_a$ w.r.t.~$\cL_k$.
It follows that $\RTD(\cL_k) \le \TD(\cL_k) \le k$.
It remains to show that $\RTD(\cL_k) \ge k$. To this end,
we consider the subclass $\cL'_k$  consisting of all 
concepts $I_a$ such that $a \in A$ and $a$ has only finitely many $1$'s 
in its binary representation $(\alpha_n)_{n\in\mathbb{N}}$, i.e., all but finitely many of the $\alpha_n$ are zero. 
Pick any concept $I_a \in \cL'_k$.
Let $T$ be any set of at most $k-1$ examples labeled 
consistently according to $I_a$. At least one of the 
positive examples $a_1,\ldots,a_k$ must be missing, 
say $a_i$ is missing. Let $J_{a,i}$ be the set of indices 
given by $J_{a,i} = \{n\in\nats_0:\ \alpha_{kn+i}=0\}$.
The following observations show that there 
exists some $a' \in \cX\sm\{a\}$ such that $I_{a'}$ is
consistent with $T$. 
\begin{itemize}
\item
When we set some (at least one but only finitely many) 
of the bits $\alpha_{kn+i}$ with $n \in J_{a,i}$ from $0$ to $1$ 
(while keeping fixed the remaining bits of the binary 
representation of $a$), then we obtain a number $a' \neq a$ 
such that $I_{a'}$ is still consistent with all positive 
examples in $T$ (including the example $(a,+)$ which might 
be in $T$).
\item
Note that $J_{a,i}$ is an infinite set. It is therefore possible
to choose the bits that are set from $0$ to $1$ in such a
fashion that the finitely many bit patterns represented
by the numbers in $T^- \cap [1,2)$ are avoided.
\item
It is furthermore possible to choose the bits that are
set from $0$ to $1$ in such a fashion that the resulting
number $a'$ is as close to $a$ as we like so that $I_{a'}$
is also consistent with the negative examples 
from $T^- \cap[0,1)$ and $a' \in A$.
\end{itemize}
It follows from this reasoning that no set with less
than $k$ examples can possibly be a teaching set for~$I_a$.
Since this holds for an arbitrary choice of $a$, we may 
conclude that 
$\RTD(\cL_k) \ge \RTD(\cL'_k) \ge \TD_{min}(\cL'_k) = k$.
\end{proof}

\section{Preference-based Teaching with Positive Examples Only} 
\label{sec:closure-operator}

The main purpose of this section is to relate
positive preference-based teaching to ``spanning sets" 
and ``closure operators", which are well-studied
concepts in the computational learning theory literature. 
Let $\cL$ be a concept class over the universe $\cX$.
We say that $S\seq\cX$ is a {\em spanning set} of $L\in\cL$ 
w.r.t.~$\cL$ if $S \seq L$ and any set in $\cL$ that 
contains $S$ must contain $L$ as well.\footnote{This 
generalizes the classical definition of a spanning 
set~\cite{HSW1990}, which is given 
w.r.t.~intersection-closed classes only.}
In other words, $L$ is the unique smallest concept in $\cL$ 
that contains $S$. We say that $S\seq\cX$ is a {\em weak spanning 
set} of $L\in\cL$ w.r.t.~$\cL$ if $S\subseteq L$ and $S$ is not 
contained in any proper subset of $L$ in $\cL$.\footnote{Weak 
spanning sets have been used in the field of recursion-theoretic 
inductive inference under the name ``tell-tale sets''~\cite{Ang1980}.} 
We denote by $I(\cL)$ (resp.~$I'(\cL)$) the smallest number $k$ 
such that every concept $L \in \cL$ has a spanning set (resp.~a 
weak spanning set) w.r.t.~$\cL$ of size at most $k$. Note that $S$ 
is a spanning set of $L$ w.r.t.~$\cL$ iff $S$ distinguishes $L$ 
from all concepts in $\cL$ except for supersets of $L$, i.e., 
iff $S$ is a positive teaching set for $L$ w.r.t.~$(\cL,\supset)$. 
Similarly, $S$ is a weak spanning set of $L$ w.r.t.~$\cL$ 
iff $S$ distinguishes $L$ from all its proper subsets in $\cL$ 
(which is necessarily the case when $S$ is a positive teaching 
set). These observations can be summarized as follows:

\begin{equation} \label{eq:span-pbtd}
I'(\cL) \le \PBTD^+(\cL) \le \PBTD^+(\cL,\supset) \le I(\cL) 
\enspace .
\end{equation}

The last two inequalities are straightforward. The inequality $I'(\cL) \le \PBTD^+(\cL)$ follows from Lemma~\ref{lem:ts-pos}, which implies that no concept $L$ can have a preference-based teaching set $T$ smaller than its smallest weak spanning set. Such a set $T$ would be consistent with some proper subset of $L$, which is impossible by Lemma~\ref{lem:ts-pos}.

Suppose $\cL$ is intersection-closed. 
Then $\cap_{L\in\cL:S \seq L}L$ is the unique smallest
concept in $\cL$ containing $S$. If $S \seq L_0$ is 
a weak spanning set of $L_0 \in \cL$, 
then $\cap_{L\in\cL:S \seq L}L = L_0$ because, on the one hand,
$\cap_{L\in\cL:S \seq L}L \seq L_0$ and, on the other hand,
no proper subset of $L_0$ in $\cL$ contains $S$.
Thus the distinction between spanning sets and weak spanning
sets is blurred for intersection-closed classes:

\begin{lemma} \label{lem:cap-closed}
Suppose that $\cL$ is intersection-closed. 
Then $I'(\cL) = \PBTD^+(\cL) = I(\cL)$.
\end{lemma}

\begin{example}
Let $\cR_d$ denote the class of $d$-dimensional axis-parallel
hyper-rectangles (= $d$-dimensio- nal boxes). This class is
intersection-closed and clearly $I(\cR_d)=2$.  
Thus $\PBTD^+(\cR_d)=2$. 
\end{example}

A mapping $\cl:2^\cX \ra 2^\cX$ is said 
to be a {\em closure operator} on the universe $\cX$
if the following conditions hold for all sets $A,B \seq \cX$:
\[
A \seq B \impl \cl(A) \seq \cl(B)\ \mbox{ and }\ 
A \seq \cl(A) = \cl(\cl(A)) \enspace .
\]
The following notions refer to an arbitrary but fixed closure 
operator. The set $\cl(A)$ is called the {\em closure} of $A$. 
A set $C$ is said to be {\em closed}
if $\cl(C) = C$. It follows 
that precisely the sets $\cl(A)$ with $A \seq \cX$ are closed. With this notation, we observe the following lemma.

\begin{lemma} \label{lem:span-closure}
Let $\cC$ be the set of all closed subsets of $\cX$ under some closure operator $\cl$, and let $L\in\cC$. If $L = \cl(S)$, then $S$ is a spanning set of $L$ w.r.t.~$\cC$.
\end{lemma}

\begin{proof}
Suppose $L'\in\cC$ and $S\subseteq L'$. Then $L = \cl(S) \seq \cl(L') = L'$.
\end{proof}

For every closed set $L \in \cL$, 
let $s_{cl}(L)$ denote the size (possibly $\infty$) of the 
smallest set $S \seq \cX$ such that $\cl(S) = L$. With this 
notation, we get the following (trivial but useful) result:

\begin{theorem} \label{th:span}
Given a closure operator, let $\cC[m]$ be the class 
of all closed subsets $C \seq \cX$ with $s_{cl}(C) \le m$. 
Then $\PBTD^+(\cC[m]) \le \PBTD^+(\cC[m],\supset) \le m$. 
Moreover, this holds with equality provided 
that $\cC[m] \sm \cC[m-1] \neq \eset$.
\end{theorem}

\begin{proof}
The inequality $\PBTD^+(\cC[m],\supset) \le m$ follows directly
from Equation~(\ref{eq:span-pbtd}) and Lemma~\ref{lem:span-closure}. \\
Pick a concept $C_0 \in \cC[m]$ such that $s_{cl}(C_0) = m$.
Then any subset $S$ of $C_0$ of size less than $m$ spans only
a proper subset of $C_0$, i.e., $\cl(S) \subset C_0$.
Thus $S$ does not distinguish $C_0$ from $\cl(S)$. However, by Lemma~\ref{lem:ts-pos}, 
any preference-based learner must strictly prefer $\cl(S)$ over $C_0$.
It follows that there is no positive teaching set of size
less than $m$ for $C_0$ w.r.t.~$\cC[m]$.
\end{proof}

Many natural classes can be cast as classes of the form $\cC[m]$
by choosing the universe and the closure operator appropriately; the following examples illustrate the usefulness of Theorem~\ref{th:span} in that regard.

\begin{example}\label{exmp:pbtdpluslinset}
Let 
\[ \LINSET_k = \{\spn{G}: (G \subset \nats) \wedge (1 \le |G| \le k)\} \]
where $\spn{G} = \left\{\sum_{g \in G}a(g)g: a(g)\in\nats_0\right\}$.
In other words, $\LINSET_k$ is the set of all non-empty linear subsets of $\mathbb{N}_0$ that are generated by at most $k$ generators. 
Note that the mapping $G \mapsto \spn{G}$ is a closure operator over 
the universe $\nats_0$. Since 
obviously $\LINSET_k \sm \LINSET_{k-1} \neq \eset$,
we obtain $\PBTD^+(\LINSET_k) = k$. 
\end{example}

\begin{example} \label{ex:polygons}
Let $\cX = \reals^2$ and let $\mathcal{C}_k$ be the class of convex polygons with
at most $k$ vertices. Defining $\cl(S)$ to be the convex closure
of $S$, we obtain $\cC[k]=\mathcal{C}_k$ and thus $\PBTD^+(\mathcal{C}_k) = k$. 
\end{example}

\begin{example} 
Let $\mathcal{X}=\mathbb{R}^n$ and let $\mathcal{C}_k$ 
be the class of polyhedral cones that can be generated 
by $k$ (or less) vectors in $\reals^n$. If we take $\cl(S)$ 
to be the conic closure of $S \seq\reals^n$, 
then $\mathcal{C}[k]=\mathcal{C}_k$ and 
thus $\PBTD^+(\mathcal{C}_k)=k$.
\end{example}

\section{A Convenient Technique for Proving Upper Bounds}
\label{sec:admissible-mappings}

In this section, we give an alternative definition of 
the preference-based teaching dimension using the notion of 
an ``admissible mapping".  Given a concept class $\cL$ over 
a universe $\cX$, let $T$ be a mapping $L \mapsto T(L) 
\seq \cX \times \{-,+\}$ that assigns a set $T(L)$ 
of labeled examples to every set $L \in \cL$ such that
the labels in $T(L)$ are consistent with $L$. The {\em order} 
of $T$, denoted as $\ord(T)$, is defined 
as $\sup_{L \in \cL}|T(L)| \in \nats\cup\{\infty\}$.
Define the mappings $T^+$ and $T^-$ by 
setting $T^+(L) = \{x : (x,+) \in T(L)\}$ 
and $T^-(L) = \{x : (x,-) \in T(L)\}$ for every $L\in\cL$.
We say that $T$ is {\em positive} if $T^-(L)=\eset$ for every $L\in\cL$.
In the sequel, we will occasionally identify a positive mapping $L \mapsto T(L)$ 
with the mapping $L \mapsto T^+(L)$. The symbol ``$+$'' as an upper index of $T$
will always indicate that the underlying mapping $T$ is positive.

\noindent
The following relation will help to clarify under which
conditions the sets $(T(L))_{L\in\cL}$ are teaching sets 
w.r.t.~a suitably chosen preference relation:
\[ 
R_T = 
\{(L,L')\in\cL\times\cL:\ (L \neq L') \wedge (\mbox{$L$ is consistent with $T(L')$})\}
\enspace .
\]
The transitive closure of $R_T$ is denoted as $\trcl(R_T)$ in the sequel.
The following notion will play an important role in this paper:

\begin{definition} \label{def:admissible}
A mapping $L \mapsto T(L)$ with $L$ ranging over all concepts in $\cL$ 
is said to be {\em admissible for $\cL$} if the following holds:
\begin{enumerate}
\item
For every $L  \in\cL$, $L$ is consistent with $T(L)$.
\item 
The relation $\trcl(R_T)$ is asymmetric (which clearly implies that $R_T$ 
is asymmetric too).
\end{enumerate} 
\end{definition}
If $T$ is admissible, then $\trcl(R_T)$ is transitive and asymmetric, i.e.,
$\trcl(R_T)$ is a strict partial order on $\cL$. We will therefore use the
notation $\prec_T$ instead of $\trcl(R_T)$ whenever $T$ is known to be 
admissible.

\begin{lemma}
Suppose that $T^+$ is a positive admissible mapping for $\cL$. Then 
the relation $\prec_{T^+}$ on $\cL$ extends the relation $\supset$ 
on $\cL$. More precisely, the following holds for all $L,L' \in \cL$:
\[ L' \subset L \impl (L,L') \in R_{T^+} \impl L \prec_{T^+} L' \enspace . \]
\end{lemma}

\begin{proof}
If $T^+$ is admissible, then $L'$ is consistent with $T^+(L')$.
Thus $T^+(L') \seq L' \subset L$ so that $L$ is consistent 
with $T^+(L')$ too. Therefore $(L,L') \in R_{T^+}$, i.e.,
$L \prec_{T^+} L'$.
\end{proof}

\noindent
The following result clarifies how admissible mappings are related
to preference-based teaching:

\begin{lemma}
For each concept class $\cL$, the following holds:
\[ 
\PBTD(\cL) = \inf_T \ord(T)\ \mbox{ and }\  
\PBTD^+(\cL) = \inf_{T^+} \ord(T^+) 
\]
where $T$ ranges over all mappings that are admissible for $\cL$
and $T^+$ ranges over all positive mappings that are admissible for $\cL$.
\end{lemma}

\begin{proof}
We restrict ourselves to the proof 
for $\PBTD(\cL) = \inf_T \ord(T)$ because the equation
$\PBTD^+(\cL) = \inf_{T^+} \ord(T^+)$ can be obtained
in a similar fashion. We first prove 
that $\PBTD(\cL)$ $\le \inf_T \ord(T)$. Let $T$ be an
admissible mapping for $\cL$. It suffices to show that, for every $L\in\cL$, 
$T(L)$ is a teaching set for $L$ w.r.t.~$(\cL,\prec_T)$. 
Suppose $L'\in\cL\sm\{L\}$ is consistent with $T(L)$. Then $(L',L) \in R_T$ 
and thus $L' \prec_T L$. It follows that $\prec_T$ prefers $L$ 
over all concepts $L'\in\cL\sm\{L\}$ that are consistent with $T(L)$. Thus $T$ is a teaching set for $L$ w.r.t.~$(\cL,\prec_T)$, as desired. 

We now prove that $\inf_T \ord(T) \le \PBTD(\cL)$. 
Let $\prec$ be a strict partial order on $\cL$ and let $T$ be a mapping
such that, for every $L\in\cL$, $T(L)$ is a teaching set for $L$
w.r.t.~$(\cL,\prec)$. It suffices to show that $T$ is admissible
for $\cL$. Consider a pair $(L',L) \in R_T$. The definition of $R_T$ implies
that $L' \neq L$ and that $L'$ is consistent with $T(L)$. Since $T(L)$ 
is a teaching set w.r.t.~$(\cL,\prec)$, it follows that $L' \prec L$.
Thus, $\prec$ is an extension of $R_T$. Since $\prec$ is transitive,
it is even an extension of $\trcl(R_T)$. Because $\prec$ is asymmetric,
$\trcl(R_T)$ must be asymmetric, too. It follows 
that $T$ is admissible.
\end{proof}

%
%

\section{Preference-based Teaching of Linear Sets} \label{sec:linsets}

Some work in computational learning theory \cite{Abe89,GSZ2015,Takada92} is concerned with learning \emph{semi-linear sets}, i.e., 
unions of linear subsets of $\mathbb{N}^k$ for some fixed $k\ge 1$, where each linear set consists of exactly those elements that can be written as the sum of some constant vector $c$ and a linear combination of the elements of some fixed set of generators, see Example~\ref{exmp:pbtdpluslinset}. While semi-linear sets are of common interest in mathematics in general, they play a particularly important role in the theory of formal languages, due to \emph{Parikh's theorem}, by which the so-called Parikh vectors of strings in a context-free language always form a semi-linear set~\cite{Parikh66}.

A recent study \cite{GSZ2015} analyzed computational teaching of classes of linear subsets of $\mathbb{N}$ (where $k=1$) and some variants thereof, as a substantially simpler yet still interesting special case of semi-linear sets.  In this section, we extend that study to preference-based teaching.

Within the scope of this section, all concept classes are formulated over the universe 
$\cX = \nats_0$. Let $G = \{g_1,\ldots$ $,g_k\}$ be a finite subset
of $\nats$. We denote by $\spn{G}$ resp.~by $\spn{G}_+$ the following
sets:
\[ 
\spn{G} = \left\{\sum_{i=1}^{k}a_ig_i:\ a_1,\ldots,a_k\in\nats_0\right\}\ 
\mbox{ and }\  
\spn{G}_+ = \left\{\sum_{i=1}^{k}a_ig_i:\ a_1,\ldots,a_k\in\nats\right\} 
\enspace .
\]

We will determine (at least approximately) the 
preference-based teaching dimension of the following concept classes 
over $\nats_0$:
\begin{eqnarray*}
\LINSET_k & = & \{\spn{G}:\ (G \subset \nats) \wedge (1 \le |G| \le k)\} \enspace . \\
\CFLINSET_k & = & \{\spn{G}:\ (G \subset \nats) \wedge (1 \le |G| \le k) \wedge (\gcd(G)=1)\} \enspace . \\
\NELINSET_k & = & \{\spn{G}_+:\ (G \subset \nats) \wedge (1 \le |G| \le k) \} \enspace . \\
\NECFLINSET_k & = & \{\spn{G}_+:\ (G \subset \nats) \wedge (1 \le |G| \le k) \wedge (gcd(G)=1)\} \enspace .
\end{eqnarray*}

A subset of $\nats_0$ whose complement in $\nats_0$ is finite is said
to be {\em co-finite}. The letters ``CF'' in $\CFLINSET$ mean ``co-finite''. 
The concepts in $\LINSET_k$ have the algebraic structure of a monoid
w.r.t.~addition. The concepts in $\CFLINSET_k$ are also known as 
``numerical semigroups''~\cite{RG-S2009}. A zero coefficient $a_j=0$ 
erases $g_j$ in the linear combination $\sum_{i=1}^{k}a_ig_i$. 
Coefficients from $\nats$ are non-erasing in this sense. 
The letters ``NE'' in ``$\NELINSET$'' mean ``non-erasing''.

The {\em shift-extension} $\cL'$ of a concept class $\cL$ 
over the universe $\nats_0$ is defined as follows:
\begin{equation} \label{eq:shift-extension}
\cL' = \{c+L:\ (c\in\nats_0) \wedge (L\in\cL)\} \enspace .
\end{equation}

The following bounds on $\RTD$ and $\RTD^+$ (for sufficiently
large values of $k$)\footnote{For instance, $\RTD^+(\LINSET_k)=\infty$
holds for all $k\ge2$ and $\RTD(\LINSET_k) = \mbox{?}$ 
(where ``?'' means ``unknown'') holds for all $k\ge4$.} 
are known from~\cite{GSZ2015}:
\[
\begin{array}{|l|l|l|}
\hline
             & \RTD^+ & \RTD \\
\hline
\LINSET_k    & =\infty & \mbox{?} \\
\CFLINSET_k  & =k      & \in\{k-1,k\} \\
\NELINSET'_k & =k+1  & \in\{k-1,k,k+1\} \\ 
\hline
\end{array}
\]
Here $\NELINSET'_k$ denotes the shift-extension
of $\NELINSET_k$ .

The following result shows the corresponding bounds 
with PBTD in place of RTD:

\begin{theorem} \label{th:bounds-linset}
The bounds in the following table are valid:
\[
\begin{array}{|l|l|l|}
\hline
             & \PBTD^+ & \PBTD \\
\hline
\LINSET_k    & =k & \in\{k-1,k\} \\
\CFLINSET_k  & =k      & \in\{k-1,k\} \\
\NELINSET_k  & 
							 \in \left[k-1:k\right] &
               \in \left[\left\lfloor\frac{k-1}{2}\right\rfloor:k\right] \\
\NECFLINSET_k & 
							 \in \left[k-1:k\right] &
               \in \left[\left\lfloor\frac{k-1}{2}\right\rfloor:k\right] \\
\hline
\end{array}
\]
Moreover
\begin{equation} \label{eq:more-bounds}
\PBTD^+(\cL') = k+1\ \wedge\ \PBTD(\cL') \in \{k-1,k,k+1\} 
\end{equation}
holds for all $\cL\in\{\LINSET_k,\CFLINSET_k,\NELINSET_k,\NECFLINSET_k\}$.
\end{theorem}

Note that the equation $\PBTD^+(\LINSET_k) = k$ was already
proven in Example~\ref{exmp:pbtdpluslinset}, using the fact that $G\mapsto\spn{G}$ is a closure operator. 
Since $G\mapsto\spn{G}_+$ is not a closure operator, we give a separate argument to prove an upper bound 
of $k$ on $\PBTD^+(\NELINSET_k)$ (see Lemma~\ref{lem:nelinset-ub} in Appendix~\ref{app:linsets}). All other upper 
bounds in Theorem~\ref{th:bounds-linset} are then easy
to derive. The lower bounds in
Theorem~\ref{th:bounds-linset} are much harder to obtain.
A complete proof of Theorem~\ref{th:bounds-linset} will
be given in Appendix~\ref{app:linsets}.

\section{Preference-based Teaching of Halfspaces} \label{sec:halfspaces}

In this section, we study preference-based teaching of halfspaces. 
We will denote the all-zeros vector as $\vec{0}$. The vector with $1$
in coordinate $i$ and with $0$ in the remaining coordinates is
denoted as $\vec{e}_i$. The dimension of the Euclidean space
in which these vectors reside will always be clear from the context.
The sign of  a real number $x$ (with value $1$ if $x>0$, value $-1$ 
if $x<0$, and value $0$ if $x=0$) is denoted by $\sign(x)$.

Suppose that $w\in\reals^d\sm\{\vec0\}$ and $b\in\reals$. 
The {\em (positive) halfspace induced by $w$ and $b$} is 
then given by
\[ 
H_{w,b} = \{x\in\reals^d:\ w^\top x + b \ge 0\} 
\enspace . 
\]
Instead of $H_{w,0}$, we simply write $H_w$.
Let $\cH_d$ denote the class of $d$-dimensional Euclidean 
half\-spaces:
\[ 
\cH_d = \{H_{w,b}:\ w\in\reals^d\sm\{\vec{0}\} \wedge b \in \reals\} 
\enspace . \]
Similarly, $\cH_d^0$ denotes the class of $d$-dimensional 
homogeneous Euclidean halfspaces:
\[ \cH_d^0 = \{H_w:\ w\in\reals^d\sm\{\vec{0}\}\} \enspace . \]
Let $S_{d-1}$ denote the $(d-1)$-dimensional unit sphere in $\reals^d$.
Moreover $S_{d-1}^+ = \{x\in S_{d-1}: x_d>0\}$ denotes the
``northern hemisphere''. If not stated explicitly otherwise, 
we will represent homogeneous halfspaces with normalized vectors 
residing on the unit sphere. We remind the reader of the following 
well-known fact:

\begin{remark} \label{rem:orthogonal-group}
The orthogonal group in dimension $d$ (i.e., the multiplicative group 
of orthogonal $(d \times d)$-matrices) acts transitively on $S_{d-1}$
and it conserves the inner product. 
\end{remark}

We now prove a helpful lemma, stating that each vector $w^*$ 
in the northern hemisphere may serve as a representative 
for some homogeneous halfspace $H_u$ in the sense that all other 
elements of $H_u$ in the northern hemisphere have a strictly 
smaller $d$-th component than $w^*$. This will later help to 
teach homogeneous halfspaces with a preference that orders vectors 
by the size of their last coordinate.

\begin{lemma} \label{lem:close2northpole}
Let $d\ge2$, let $0<h\le1$ and 
let $R_{d,h} = \{w \in S_{d-1}: w_d = h\}$.
With this notation the following holds. 
For every $w^* \in R_{d,h}$, there exists $u\in\reals^d\sm\{\vec{0}\}$
such that
\begin{equation} \label{eq:good-choice}
(w^* \in H_u) \wedge 
(\forall w \in (S_{d-1}^+ \cap H_u)\sm\{w^*\}: w_d < h) 
\enspace .
\end{equation}
\end{lemma}

\begin{proof}
For $h=1$, the statement is trivial, since $R_{d,1} = \{\vec{e}_d\}$. 
So let $h<1$.

Because of Remark~\ref{rem:orthogonal-group}, we may assume
without loss of generality that the vector $w^* \in R_{d,h}$
equals $(0,\ldots,0,\sqrt{1-h^2},h)$. It suffices therefore
to show that, with this choice of $w^*$, 
the vector $u = (0,\ldots,0,w_d^*,-w_{d-1}^*)$ 
satisfies~(\ref{eq:good-choice}). Note that $w \in H_{u}$ 
iff $\spn{u,w} = w_d^*w_{d-1} - w_{d-1}^*w_d \ge 0$. 
Since $\spn{u,w^*}=0$, we have $w^* \in H_u$. Moreover, 
it follows that
\[ 
S_{d-1}^+ \cap H_{u} = 
\left\{w\in S_{d-1}^+: 
\frac{w_{d-1}}{w_d} \ge \frac{w_{d-1}^*}{w_d^*} > 0 \right\}
\enspace .
\]
It is obvious that no vector $w \in S_{d-1}^+ \cap H_u$
can have a $d$-th component $w_d$ exceeding $w_d^*=h$ and 
that setting $w_d=h=w_d^*$ forces the 
settings $w_{d-1}=w_{d-1}^*=\sqrt{1-h^2}$ 
and $w_1 = \ldots = w_{d-2} = 0$. 
Consequently,~(\ref{eq:good-choice}) is satisfied, which 
concludes the proof. 
\end{proof}

With this lemma in hand, we can now prove an upper bound of 2 
for the preference-based teaching dimension of the class of 
homogeneous halfspaces, independent of the underlying dimension~$d$.

\begin{theorem} \label{th:halfspace0-ub}
$\PBTD(\cH_1^0) = \TD(\cH_1^0) = 1$ and, for every $d\ge2$, 
we have $\PBTD(\cH_d^0)$ $\le 2$.
\end{theorem}

\begin{proof}
Clearly, $\PBTD(\cH_1^0) = \TD(\cH_1^0) = 1$ since $\cH_1^0$
consists of the two sets $\{x\in\reals: x\ge0\}$ 
and $\{x\in\reals: x\le0$\}. 

\noindent
Suppose now that $d\ge2$. Let $w^*$ be the target weight vector 
(i.e., the weight vector that has to be taught). Under the following 
conditions, we may assume without loss of generality that $w_d^*\neq0$:
\begin{itemize}
\item
For any $0<s_1<s_2$, the student prefers any weight vector 
that ends with $s_2$ zero coordinates over any weight vector
that ends with only $s_1$ zero coordinates.
\item
If the target vector ends with (exactly) $s$ zero coordinates, 
then the teacher presents only examples ending with (at least) 
$s$ zero coordinates.
\end{itemize}
In the sequel, we specify a student and a teacher such that
these conditions hold, so that we will consider only target 
weight vectors $w^*$ with $w_d^*\neq0$.

\noindent
The student has the following preference relation:
\begin{itemize}
\item
Among the weight vectors $w$ with $w_d\neq0$, 
the student prefers vectors with larger values of $|w_d|$ 
over those with smaller values of $|w_d|$.
\end{itemize}

\noindent
The teacher will use two examples. The first one is chosen as
\[ 
\left\{ \begin{array}{ll}
     (-\vec{e}_d,-) & \mbox{if $w^*_d>0$} \\
     (\vec{e}_d,-)  &\mbox{if $w^*_d<0$} 
        \end{array} \right. \enspace .
\]
This example reveals whether the unknown weight 
vector $w^* \in S_{d-1}$ has a strictly positive 
or a strictly negative $d$-th component.
For reasons of symmetry, we may assume that $w_d^*>0$.
We are now precisely in the situation that is described
in Lemma~\ref{lem:close2northpole}. Given $w^*$ and $h=w_d^*$, 
the teacher picks as a second example $(u,+)$ 
where $u\in\reals^d\sm\{\vec{0}\}$ has the properties
described in the lemma. It follows 
immediately that the student's preferences will make her
choose the weight vector $w^*$.
\end{proof}

The upper bound of 2 given in Theorem~\ref{th:halfspace0-ub} 
is tight, as is stated in the following lemma.

\begin{lemma} \label{lem:halfspace-neg}
For every $d\ge2$, we have $\PBTD(\cH_d^0)\ge 2$.
\end{lemma}

\begin{proof}
We verify this lemma via Lemma~\ref{lem:lb-tdmin}, by providing 
a finite subclass $\cF$ of $\cH_2^0$ such that $\TD_{min}(\cF)=2$. 
Let $\cF=\{H_w : \vec{0} \neq w\in\{-1,0,1\}^2\}$. It is easy to 
verify that each of the $8$ halfspaces in $\cF$ has a teaching 
dimension of 2 with respect to $\cF$. This example can be extended 
to higher dimensions in the obvious way.
\end{proof}

We thus conclude that the class of homogeneous halfspaces 
has a preference-based teaching dimension of 2, independent 
of the dimensionality $d\ge 2$.

\begin{corollary}
For every $d\ge2$, we have $\PBTD(\cH_d^0)=2$.
\end{corollary}

By contrast, we will show next that the recursive teaching 
dimension of the class of homogeneous halfspaces grows with 
the dimensionality. 

\begin{theorem}\label{thm:tdrtdhalfspace}
For any $d \ge 2$, $\TD(\mathcal{H}^0_d) = \RTD(\mathcal{H}^0_d) = d+1$.
\end{theorem}

\begin{proof}
Assume by normalization that the target weight vector has norm $1$,
i.e., it is taken from $S_{d-1}$. Remark~\ref{rem:orthogonal-group} 
implies that all weight vectors in $S_{d-1}$ are equally hard to teach. 
It suffices therefore to show that $\TD(H_{\vec{e}_1},\cH^0_d) = d+1$.

We first show that $\TD(H_{\vec{e}_1},\cH^0_d) \le d+1$.
Define $u = -\sum_{i=2}^{d}\vec{e}_i$. We claim 
that $T = \{ (\vec{e}_i,+): 2 \leq i \leq d\} \cup \{(u,+),(\vec{e}_1,+)\}$
is a teaching set for $H_{\vec{e}_1}$ w.r.t.~$\cH^0_d.$ 
Consider any $w \in S_{d-1}$ such that $H_w$ is consistent 
with $T$. Note that $w_i = \spn{\vec{e}_i,w} \geq 0$ for 
all $i \in \{2,\ldots,d\}$ 
and $\spn{u,w} = -\sum_{i=2}^{d}w_i \geq 0$ together
imply that $w_i = 0$ for all $i\in\{2,\ldots,d\}$
and therefore $w = \pm \vec{e}_1$. Furthermore, 
$w_1 = \spn{w,\vec{e}_1} \geq 0$, and so $w = \vec{e}_1$, as required.

Now we show that $\TD(H_{\vec{e}_1},\cH^0_d) \geq d+1$ holds
for all $d\ge2$. It is easy to see that two examples do not 
suffice for distinguishing $\vec{e}_1\in\reals^2$ from all weight
vectors in $S_1$. In other words, $\TD(H_{\vec{e}_1},\cH^0_2)  \ge3$.
Suppose now that $d\ge3$. It is furthermore easy to see that 
a teaching set $T$ which distinguishes $\vec{e}_1$ from all weight
vectors in $S_{d-1}$ must contain at least one positive 
example $u$ that is orthogonal to $\vec{e}_1$. The 
inequality $\TD(H_{\vec{e}_1},\cH^0_d) \geq d+1$ is now obtained 
inductively because the example $(u,+) \in T$ 
leaves open a problem that is not easier than teaching $\vec{e}_1$ 
w.r.t.~the $(d-2)$-dimensional sphere $\{x \in S_{d-1}: x \perp u\}$.
\end{proof}  

We have thus established that the class of homogeneous halfspaces 
has a recursive teaching dimension growing linearly with $d$, 
while its preference-based teaching dimension is constant. 
In the case of general (i.e., not necessarily homogeneous) 
$d$-dimensional halfspaces, the difference between $\RTD$ 
and $\PBTD$ is even more extreme. 
On the one hand, by generalizing the proof of Lemma~\ref{lem:huge-gap}, 
it is easy to see that $\RTD(\cH_d)=\infty$ for all $d\ge 1$. On the 
other hand, we will show in the remainder of this section  
that $\PBTD(\cH_d) \le 6$, independent of the value of $d$. 

We will assume in the sequel (by way of normalization) that 
an inhomogeneous halfspace has a bias $b\in\{\pm1\}$.
We start with the following result:

\begin{lemma} \label{lem:3-examples}
Let $w^* \in \reals^d$ be a vector with a non-trivial $d$-th 
component $w^*_d\neq0$ and let $b^*\in\{\pm1\}$ be a bias. 
Then there exist three examples labeled according to $H_{w^*,b^*}$
such that the following holds. Every weight-bias pair $(w,b)$ 
consistent with these examples 
satisfies $b=b^*$, $\sign(w_d) = \sign(w_d^*)$ and
\begin{equation} \label{eq:w-d-constraint} 
\left\{ \begin{array}{ll}
         |w_d| \ge |w_d^*| & \mbox{if $b^*=-1$} \\
         |w_d| \le |w_d^*| & \mbox{if $b^*=+1$}
       \end{array} \right. \enspace .
\end{equation}
\end{lemma}

\begin{proof}
Within the proof, we use the label ``$1$'' instead of ``$+$''
and the label ``$-1$''  instead of ``$-$''. 
The pair $(w,b)$ denotes the student's hypothesis 
for the target weight-bias pair $(w^*,b^*)$. 
The examples shown to the student will involve the unknown 
quantities $w^*$ and $b^*$. Each example will lead to a new 
constraint on $w$ and $b$. We will see that the collection 
of these constraints reveals the required information.
We proceed in three stages:
\begin{enumerate}
\item
The first example is chosen as $(\vec{0},b^*)$.
The pair $(w,b)$ can be consistent with this example only
if $b = -1$ in the case that $b^*=-1$ and $b\in\{0,1\}$
in the case that $b^*=1$.
\item
The next example is chosen 
as $\vec{a}_2 = -\frac{2b^*}{w_d^*}\cdot\vec{e}_d$
and labeled ``$-b^*$''. 
Note that $\spn{w^*,\vec{a}_2}+b^* = -b^*$.
We obtain the following new constraint:
\[ 
\spn{w,\vec{a}_2}+b = 
\left\{ \begin{array}{ll} 
         -2\frac{w_d}{w_d^*}+\overbrace{b}^{\in\{0,1\}}  <  0 & \mbox{if $b^*=1$} \\
         +2\frac{w_d}{w_d^*}+\underbrace{b}_{=-1} \ge 0 & \mbox{if $b^*=-1$}
        \end{array} \right.  \enspace .
\]
The pair $(w,b)$ with $b=b^*$ if $b^*=-1$ and $b\in\{0,1\}$ 
if $b^*=1$ can satisfy the above constraint only
if the sign of $w_d$ equals the sign of $w_d^*$. 
\item
The third example is chosen as
the example $\vec{a}_3 = -\frac{b^*}{w_d^*}\cdot\vec{e}_d$ 
with label ``$1$''. 
Note that $\spn{w^*,\vec{a}_3}^*+b^* = 0$. 
We obtain the following new constraint: 
\[ 
\spn{w,\vec{a}_3} = -\frac{b^*w_d}{w_d^*}+b \ge 0 \enspace .
\]
Given that $w$ is already constrained to weight vectors
satisfying $\sign(w_d)=\sign($ $w_d^*)$, we can safely 
replace $w_d/w_d^*$ by $|w_d|/|w_d^*|$. This yields
$|w_d|/|w_d^*| \le b$ if $b^*=1$ and $|w_d|/|w_d^*|\ge -b$
if $b^*=-1$. Since $b$ is already constrained as described
in stage 1 above, we obtain $|w_d|/|w_d^*| \le b \in \{0,1\}$
if $b^*=1$ and $|w_d|/|w_d^*|\ge -b = 1$ if $b^*=-1$.
The weight-bias pair $(w,b)$ satisfies these constraints
only if $b=b^*$ and if~(\ref{eq:w-d-constraint}) is valid. 
\end{enumerate}
The assertion of the lemma is immediate from this discussion.
\end{proof}

\begin{theorem} \label{th:halfspace-ub}
$\PBTD(\cH_d) \le 6$.
\end{theorem}

\begin{proof}
As in the proof of Lemma~\ref{lem:3-examples}, we use the 
label ``$1$'' instead of ``$+$'' and the label ``$-1$''  
instead of ``$-$''. As in the proof of Theorem~\ref{th:halfspace0-ub}, 
we may assume without loss of generality that the target weight 
vector $w^*\in\reals^d$ satisfies $w_d^* \neq 0$. The proof will
proceed in stages. On the way, we specify six rules which
determine the preference relation of the student.

{\bf Stage 1} is concerned with teaching homogeneous halfspaces 
given by $w^*$ (and $b^*=0$). The student respects the following 
rules:
\begin{description}
\item[Rule 1:]
She prefers any pair $(w,0)$ over any pair $(w',b)$ with $b\neq0$.
In other words, any homogeneous halfspace is preferred over any
non-homogeneous halfspace.
\item[Rule 2:]
Among homogeneous halfspaces, her preferences are the same as the 
ones that were used within the proof of Theorem~\ref{th:halfspace0-ub} 
for teaching homogeneous halfspaces.
\end{description}
Thus, if $b^*=0$, then we can simply apply the teaching protocol 
for homogeneous halfspaces. In this case, $w^*$ can be taught 
at the expense of only two examples. 

Stage~1 reduces the problem to teaching inhomogeneous halfspaces
given by $(w^*,$ $b^*)$ with $b^*\neq0$. We assume, by way of 
normalization, that $b^*\in\{\pm1\}$, but note that $w^*$ 
can now not be assumed to be of unit (or any other fixed)
length. 

In {\bf stage 2}, the teacher presents three examples in accordance
with Lemma~\ref{lem:3-examples}. It follows that the student will
take into consideration only weight-bias pairs $(w,b)$ such that 
the constraints $b=b^*$, $\sign(w_d) = \sign(w_d^*)$ 
and~(\ref{eq:w-d-constraint}) are satisfied. The following rule
will then induce the constraint $w_d=w_d^*$:
\begin{description}
\item[Rule 3:]
Among the pairs $(w,b)$ such that $w_d\neq0$ and $b\in\{\pm1\}$,
the student's preferences are as follows. If $b=-1$ (resp.~$b=1$), 
then she prefers vectors $w$ with a smaller (resp.~larger) value 
of $|w_d|$ over those with a larger (resp.~smaller) value of $|w_d|$.
\end{description}
Thanks to Lemma~\ref{lem:3-examples} and thanks to Rule 3,
we may from now on assume that $b=b^*$ and $w_d=w_d^*$.
In the sequel, let $w^*$ be decomposed according 
to $w^* = (\vec{w}_{d-1}^*,w_d^*) \in \reals^{d-1}\times\reals$. 
We think of $\vec{w}_{d-1}$ as the student's hypothesis 
for $\vec{w}_{d-1}^*$.

{\bf Stage~3} is concerned with the special case 
where $\vec{w}_{d-1}^*=\vec{0}$. The student will 
automatically set $\vec{w}_{d-1}=\vec{0}$ if
we add the following to the student's rule system:
\begin{description}
\item[Rule 4:]
Given that the values for $w_d$ and $b$ have been fixed already
(and are distinct from $0$), the student prefers weight-bias
pairs with $\vec{w}_{d-1}=\vec{0}$ over any weight-bias pair 
with $\vec{w}_{d-1}\neq\vec{0}$.
\end{description}

Stage 3 reduces the problem to teaching $(w^*,b^*)$ with 
fixed non-zero values for $w_d$ and $b^*$ (known to the student) 
and with $\vec{w}_{d-1}^* \neq \vec{0}$. Thus, essentially,
only $\vec{w}_{d-1}^*$ has still to be taught. In the next
stage, we will argue that the problem of teaching $\vec{w}_{d-1}^*$
is equivalent to teaching a homogeneous halfspace.

In {\bf stage 4}, the teacher will present only examples $a$ such 
that $a_d = -\frac{b^*}{w_d^*}$ so that the contribution of 
the $d$-th component to the inner product of $w^*$ and $a$ 
cancels with the bias $b^*$. Given this commitment for $a_d$,
the first $d-1$ components of the examples can be chosen so as
to teach the homogeneous halfspace $H_{\vec{w}_{d-1}^*}$.
According to Theorem~\ref{th:halfspace0-ub}, this can be 
achieved at the expense of two more examples. Of course
the student's preferences must match with the preferences 
that were used in the proof of this theorem:
\begin{description}
\item[Rule 5:] 
Suppose that the values of $w_d$ and $b$ have been fixed
already (and are distinct from $0$) and suppose 
that $\vec{w}_{d-1}\neq\vec{0}$. Then the preferences
for the choice of $\vec{w}_{d-1}$ match with the preferences
that were used in the protocol for teaching homogeneous
halfspaces.
\end{description}

After stage 4, the student takes into consideration only
weight-bias pairs $(w,b)$ such that $w_d=w_d^*$, $b=b^*$
and $H_{\vec{w}_{d-1}} = H_{\vec{w}_{d-1}^*}$.
However, since we had normalized the bias and not the weight
vector, this does not necessarily mean 
that $\vec{w}_{d-1} = \vec{w}_{d-1}^*$. On the other hand, 
the two weight vectors already coincide modulo a positive 
scaling factor, say 
\begin{equation} \label{eq:scaling}
\vec{w}_{d-1} = s \cdot \vec{w}_{d-1}^* \mbox{ for some $s>0$}
\enspace .
\end{equation} 
In order to complete the proof, it suffices to teach 
the $L_1$-norm of $\vec{w}_{d-1}^*$ to the student
(because~(\ref{eq:scaling}) 
and $\|\vec{w}_{d-1}\|_1 = \|\vec{w}_{d-1}^*\|_1$
imply that $\vec{w}_{d-1} = \vec{w}_{d-1}^*$). 
The next (and final) stage serves precisely this purpose.

As for {\bf stage~5}, we first fix some notation.
For $i=1,\ldots,k-1$, let $\beta_i = \sign(w_i^*)$. 
Note that~(\ref{eq:scaling}) implies 
that $\beta_i = \sign(w_i)$.
Let $L = \|\vec{w}_{d-1}^*\|_1$ denote the $L_1$-norm 
of $\vec{w}_{d-1}^*$. The final example is chosen 
as $\vec{a}_6 = (\beta_1,\ldots,\beta_{d-1},-(L+b^*)/w_d^*)$
and labeled ``$1$''. Note that 
\[ 
\spn{w^*,\vec{a_6}} + b^* = |w_1^*|+\ldots+|w_{d-1}^*|-L = 0
\enspace .
\]
Given that $\beta_i=\sign(w_i)$, $w_d=w_d^*$ and $b=b^*$,
the student can derive from $\vec{a_6}$ and its label the 
following constraint on~$\vec{w}_{d-1}$:
\[ 
\spn{w,\vec{a_6}} + b = |w_1|+\ldots+|w_{d-1}| - L \ge 0
\enspace .
\] 
In combination with the following rule, we can now force
the constraint $\|\vec{w}_{d-1}\|_1 = L$:
\begin{description}
\item[Rule 6:]
Suppose that the values of $w_d$ and $b$ have been fixed
already (and are distinct from $0$) and suppose
that $H_{\vec{w}_{d-1}}$ has already been fixed.
Then, among the vectors representing $H_{\vec{w}_{d-1}}$,
the ones with a smaller $L_1$-norm are preferred over the
ones with a larger $L_1$-norm.
\end{description}
An inspection of the six stages reveals that at most
six examples altogether were shown to the student
(three in stage 2, two in stage 4, and one in stage 5).
This completes the proof of the theorem. \\ \mbox{}
\end{proof}

Note that Theorems~\ref{th:halfspace0-ub} and~\ref{th:halfspace-ub}
remain valid when we allow $w$ to be the all-zero vector, which 
extends $\mathcal{H}_d^0$ by $\{\bR^d\}$ and $\mathcal{H}_d$ 
by $\{\bR^d,\emptyset\}$. $\bR^d$ will be taught with a single 
positive example, and $\emptyset$ with a single negative example.  
The student will give the highest preference to $\bR^d$, the second highest 
to $\emptyset$, and among the remaining halfspaces, the student's preferences 
stay the same.

\section{Classes with $\PBTD$ or $\PBTD^+$ Equal to One}\label{sec:pbtd1}

In this section, we will give complete characterizations of 
(i) the concept classes with a positive preference-based 
teaching dimension of $1$, and (ii) the concept classes 
with a preference-based teaching dimension of $1$.
Throughout this section, we use the label ``$1$'' to
indicate positive examples and the label ``$0$'' to
indicate negative examples. 

Let $I$ be a (possibly infinite) index set. We will
consider a mapping $A: I \times I \ra \{0,1\}$ as a
binary matrix $A \in \{0,1\}^{I \times I}$. 
$A$ is said to be {\em lower-triangular} if there 
exists a linear ordering $\prec$ on $I$ such that $A(i,i')=0$ 
for every pair $(i,i')$ such that $i \prec i'$. 

We will occasionally identify a set $L\seq\cX$ with
its indicator function by setting $L(x) = \Ind_{[x \in L]}$.

\noindent
For each $M \seq \cX$, we define
\[
M \oplus L=(L \sm M) \cup (M \sm L)
\]
and
\[ 
M \oplus \cL = \{M \oplus L: L\in\cL\}  \enspace . \]
For $T \seq \cX\times\{0,1\}$, we define similarly
\[ 
M \oplus T = 
\{(x,\bar y): (x,y) \in T\mbox{ and } x \in M\} \cup
\{(x,y) \in T: x \notin M\} \enspace .
\]
Moreover, given $M\seq\cX$ and a linear ordering $\prec$
on $\cL$, we define a linear ordering $\prec_M$ 
on $M \oplus \cL$ as follows:
\[ 
M \oplus L' \prec_M M \oplus L \dund
\underbrace{M \oplus (M \oplus L')}_{=L'} \prec
\underbrace{M \oplus (M \oplus L)}_{=L} \enspace .
\]

\begin{lemma} \label{lem:flip}
With this notation, the following holds.
If the mapping $\cL \ni L \mapsto T(L) \seq\cX\times\{0,1\}$ 
assigns a teaching set to $L$ w.r.t.~$(\cL,\prec)$, then
the mapping 
$M \oplus \cL \ni M \oplus L \mapsto M \oplus T(L) \seq \cX\times\{0,1\}$ 
assigns a teaching set to $M \oplus L$ w.r.t.~$(M\oplus\cL,\prec_M)$.
\end{lemma}
Since this result is rather obvious, we skip its proof.

We say that $\cL$ and $\cL'$ are {\em equivalent} 
if $\cL' = M\oplus\cL$ for some $M\seq\cX$ (and this 
clearly is an equivalence relation). As an immediate 
consequence of Lemma~\ref{lem:flip}, we obtain the 
following result:

\begin{lemma}
If $\cL$ is equivalent to $\cL'$, then $\PBTD(\cL) = \PBTD(\cL')$.
\end{lemma}


The following lemma provides a necessary condition for a concept class to have a preference-based teaching dimension of one.

\begin{lemma} \label{lem:single-occurrence}
Suppose that $\cL \seq2^\cX$ is a concept class of $\PBTD$ $1$.
Pick a linear ordering $\prec$ on $\cL$ and a 
mapping $\cL \ni L \mapsto (x_L,y_L) \in \cX\times\{0,1\}$
such that, for every $L\in\cL$, $\{(x_L,y_L)\}$ is a teaching
set for $L$ w.r.t.~$(\cL,\prec)$. Then
\begin{itemize}
\item 
either every instance $x\in\cX$ occurs at most once in $(x_L)_{L\in\cL}$ 
\item 
or there exists a concept $L^*\in\cL$ that is preferred over 
all other concepts in $\cL$ and $x_{L^*}$ is the only 
instance from $\cX$ that occurs twice in $(x_L)_{L\in\cL}$.
\end{itemize}
\end{lemma}

\begin{proof}
Since the mapping $T$ must be injective, no instance can occur
twice in $(x_L)_{L\in\cL}$ with the same label. Suppose that 
there exists an instance $x\in\cX$ and concepts $L \prec L^*$ 
such that $x = x_L = x_{L^*}$ and, w.l.o.g., $y_L=1$ and $y_{L^*}=0$.
Since $\{(x,1)\}$ is a teaching set for $L$ w.r.t.~$(\cL,\prec)$,
every concept $L' \succ L$ (including the ones that are preferred
over $L^*$) must satisfy $L'(x) = 0$. For analogous reasons, 
every concept $L' \succ L^*$ (if any) must satisfy $L'(x)=1$.
A concept $L'\in\cL$ that is preferred over $L^*$ would have 
to satisfy $L'(x)=0$ and $L'(x)=1$, which is impossible.
It follows that there can be no concept that is preferred over~$L^*$. 
\end{proof}

The following result is a consequence of Lemmas~\ref{lem:flip}
and~\ref{lem:single-occurrence}.

\begin{theorem} \label{th:pbtd1-equivalence}
If $\PBTD(\cL)=1$, then there exists a concept class $\cL'$
that is equivalent to $\cL$ and satisfies $\PBTD(\cL')=\PBTD^+(\cL')=1$.
\end{theorem}

\begin{proof}
Pick a linear ordering $\prec$ on $\cL$ and, for every $L\in\cL$, 
a pair $(x_L,y_L) \in \cX\times\{0,1\}$ such 
that $T(L)=\{(x_L,y_L)\}$ is a teaching set for $L$ 
w.r.t.~$(\cL,\prec)$. 
\begin{description}
\item[Case 1:]
Every instance $x\in\cX$ occurs at most once in $(x_L)_{L\in\cL}$. \\
Then choose $M = \{x_L: y_L=0\}$ and apply Lemma~\ref{lem:flip}.
\item[Case 2:] There exists a concept $L^*\in\cL$ that is preferred 
over all other concepts in $\cL$ and $x_{L^*}$ is the only
instance from $\cX$ that occurs twice in $(x_L)_{L\in\cL}$. \\
Then choose $M = \{x_L: y_L=0 \wedge L \neq L^*\}$ and apply 
Lemma~\ref{lem:flip}. With this choice, we 
obtain 
$M \oplus T(L) = \{(x_L,1)\}$ for every $L\in\cL\sm\{L^*\}$.
Since $L^*$ is preferred over all other concepts in $\cL$, 
we may teach $L^*$ w.r.t.~$(\cL,\prec)$ by the empty set 
(instead of employing a possibly $0$-labeled example).
\end{description}
The discussion shows that there is a class $\cL'$ that is equivalent
to $\cL$ and can be taught in the preference-based model with
positive teaching sets of size $1$ (or size $0$ in case of $L^*$).
\end{proof}

We now have the tools required for characterizing the concept
classes whose positive PBTD equals $1$.

\begin{theorem} \label{th:pbtd-plus1}
$\PBTD^+(\cL)=1$ if and only if there exists a 
mapping $\cL \ni L \mapsto x_L \in \cX$ such that the
matrix $A\in\{0,1\}^{(\cL\sm\{\eset\})\times(\cL\sm\{\eset\})}$ 
given by $A(L,L') = L'(x_L)$ is lower-triangular.
\end{theorem}

\begin{proof}
Suppose first that $\PBTD^+(\cL)=1$. Pick a linear ordering $\prec$ 
on $\cL$ and, for every $L\in\cL\sm\{\eset\}$, pick $x_L\in\cX$ such 
that $\{x_L\}$ is a positive teaching set 
for $L$ w.r.t.~$(\cL,\prec)$.\footnote{Such an $x_L$ always exists, 
even if $\emptyset$ is a teaching set for $L$, because every superset
of a teaching set for $L$ that is still consistent with $L$ is still 
a teaching set for $L$, cf.\ the discussion immediately after 
Lemma~\ref{lem:extension}.} 
If $L \prec L'$ (so that $L'$ is preferred over $L$), 
we must have $L'(x_L)=0$. It follows that the matrix $A$,
as specified in the theorem, is lower-triangular. 

Suppose conversely that there exists a 
mapping $\cL \ni L \mapsto x_L \in \cX$ such that the
matrix $A\in\{0,1\}^{(\cL\sm\{\eset\})\times(\cL\sm\{\eset\})}$ 
given by $A(L,L') = L'(x_L)$ is lower-triangular, say 
w.r.t.~the linear ordering $\prec$ on $\cL\sm\{\eset\}$. 
Then, for every $L\in\cL\sm\{\eset\}$, the singleton $\{x_L\}$ 
is a positive teaching set for $L$ w.r.t.~$(\cL,\prec)$ 
because it distinguishes $L$ from $\eset$ (of course) 
and also from every concept $L' \in \cL\sm\{\eset\}$
such that $L' \succ L$. If $\eset \in\cL$, then extend 
the linear ordering $\prec$ by preferring $\eset$ 
over every other concept from $\cL$ (so that $\eset$
is a positive teaching set for $\eset$ w.r.t.~$(\cL,\prec)$).
\end{proof}

In view of Theorem~\ref{th:pbtd1-equivalence}, 
Theorem~\ref{th:pbtd-plus1} characterizes every class $\cL$
with $\PBTD(\cL)=1$ up to equivalence.

\medskip
Let $\Sg(\cX) = \{\{x\}: x \in \cX\}$ denote the class
of singletons over $\cX$ and suppose that $\Sg(\cX)$
is a sub-class of $\cL$ and $\PBTD(\cL)=1$. We
will show that only fairly trivial extensions of $\Sg(\cX)$
with a preference-based dimension of $1$ are possible.

\begin{lemma} \label{lem:pbtd1-implications}
Let $\cL\seq2^\cX$ be a concept class of $\PBTD$ $1$ that
contains $\Sg(\cX)$. Let $T$ be an admissible mapping for $\cL$
that assigns a labeled example $(x_L,y_L)\in\cX\times\{0,1\}$
to each $L\in\cL$. For $b=0,1$, 
let $\cL^b = \{L\in\cL: y_L=b\}$. Similarly, 
let $\cX^b = \{x\in\cX: y_{\{x\}}\in\cL^b\}$. With this notation, 
the following holds:
\begin{enumerate}
\item
If $L\in\cL^1$ and $L \subset L' \in \cL$, then $L'\in\cL^1$.
\item
If $L'\in\cL^0$ and $L' \supset L \in \cL$, then $L\in\cL^0$.
\item
$|\cX^0| \le 2$. Moreover if $|\cX^0|=2$, then there
exist $q \neq q' \in\cX$ such that $\cX^0=\{q,q'\}$
and $x_{\{q\}} = q'$.
\end{enumerate}
\end{lemma}

\begin{proof}
Recall that $R_T = 
\{(L,L')\in\cL\times\cL:\ (L \neq L') 
\wedge (\mbox{$L$ is consistent with $T(L')$})\}$
and that $R_T$ (and even the transitive closure of $R_T$)
is asymmetric if $T$ is admissible. 
\begin{enumerate}
\item
If $L \in \cL^1$ and $L \subset L'$, then $y_L=1$ so that $L'$ 
is consistent with the example $(x_L,y_L)$. It follows 
that $(L',L) \in  R_T$. $L' \in \cL^0$ would similarly imply 
that $(L,L') \in R_T$ so that $R_T$ would not be asymmetric. 
This is in contradiction with the admissibility of $T$.
\item
The second assertion in the lemma is a logically equivalent
reformulation of the first assertion.
\item
Suppose for the sake of contradiction that $\cX^0$ contains
three distinct points, say $q_1,q_2,q_3$. Since, for $i=1,2,3$, 
$T$ assigns a $0$-labeled example to $\{q_i\}$, at least one
of the remaining two points is consistent with $T(\{q_i\})$.
Let $G$ be the digraph with the nodes $q_1,q_2,q_3$
and with an edge from $q_j$ to $q_i$ iff $\{q_j\}$ is 
consistent with $T(\{q_i\})$. Then each of the three
nodes has an indegree of at least $1$. Digraphs of this form
must contain a cycle so that $\trcl(R_T)$ is not asymmetric. This 
is in contradiction with the admissibility of $R_T$. 

A similar argument holds if $\cX^0$ contains only two distinct 
elements, say $q$ and $q'$. If neither  $x_{\{q\}} = q'$ 
nor $x_{\{q'\}} = q$, then $(\{q'\},\{q\}) \in R_T$ 
and $(\{q\},\{q'\}) \in R_T$ so that $R_T$ is not asymmetric 
--- again a contradiction to the admissibility of $R_T$. 
\end{enumerate}
\end{proof}

\noindent
We are now in the position to characterize those classes of $\PBTD$ 
one that contain all singletons.

\begin{theorem} \label{th:pbtd1}
Suppose that $\cL\seq2^\cX$ is a concept class that
contains $\Sg(\cX)$. Then $\PBTD(\cL)=1$ if and only if
the following holds. Either $\cL$ coincides with $\Sg(\cX)$
or $\cL$ contains precisely one additional concept, which is either the empty set or a set of size $2$.
\end{theorem}

\begin{proof}
We start with proving ``$\La$''. It is well known that
$\PBTD^+(\cL)=1$ for $\cL = \Sg(\cX)\cup\{\eset\}$:
prefer $\eset$ over any singleton set, set $T(\eset)=\eset$ 
and, for every $x\in\cX$, set $T(\{x\})=\{(x,1)\}$. In a
similar fashion, we can show that $\PBTD(\cL)=1$ 
for $\cL = \Sg(\cX)\cup\{\{q,q'\}\}$ for any choice 
of $q \neq q' \in \cX$. Prefer $\{q,q'\}$ over $\{q\}$
and $\{q'\}$, respectively. Furthermore, prefer $\{q\}$
and $\{q'\}$ over all other singletons. Finally, set
$T(\{q,q'\})=\eset$, $T(\{q\})=\{(q',0)\}$, $T(\{q'\})=\{(q,0)\}$
and, for every $x\in\cX\sm\{q,q'\}$, set $T(\{x\}) = \{(x,1)\}$. 

As for the proof of ``$\Ra$'', we make use of the notions
$T,x_L,y_L,\cL^0,\cL^1,\cX^0,\cX^1$ that had been introduced 
in Lemma~\ref{lem:pbtd1-implications} and we proceed by case 
analysis.

\begin{description}
\item[Case 1:] $\cX^0 = \eset$. \\
Since $\cX^0 = \eset$, we have $\cX=\cX^1$.
In combination with the first assertion in 
Lemma~\ref{lem:pbtd1-implications}, it follows that
$\cL\sm\{\eset\} = \cL^1$. We claim that no concept 
in $\cL$ contains two distinct elements. Assume for 
the sake of contradiction that there is a concept $L\in\cL$ 
such that $|L|\ge2$. It follows that, for every $q \in L$, 
$x_{\{q\}}=q$ and $y_{\{q\}}=1$ so that $(L,\{q\}) \in R_T$.
Moreover, there exists $q_0 \in L$ such that $x_L=q_0$ and $y_L=1$.
It follows that $(\{q_0\},L) \in R_T$, which contradicts the
fact that $R_T$ is asymmetric.
\item[Case 2:] $\cX^0 = \{q\}$ for some $q\in\cX$. \\
Set $q' = x_{\{q\}}$ and note that $y_{\{q\}}=0$. Moreover, 
since $\cX^1 = \cX\sm\{q\}$, we have $x_{\{p\}}=p$ and $y_{\{p\}}=1$
for every $p\in\cX\sm\{q\}$. We claim that $\cL$ cannot
contain a concept $L$ of size at least $2$ that contains
an element of $\cX\sm\{q,q'\}$. Assume for the sake of
contradiction, that there is a set $L$ such that $|L|\ge2$
and $p \in L$ for some $p\in\cX\sm\{q,q'\}$. 
The first assertion in Lemma~\ref{lem:pbtd1-implications} 
implies that $y_L=1$ (because $y_{\{p\}}=1$ and $\{p\} \seq L$). 
Since all pairs $(x,1)$ with $x \neq q$ are already in 
use for teaching the corresponding singletons, 
we may conclude that $q \in L$ and $T(L) = \{(q,1)\}$.
This contradicts the fact that $\trcl(R_T)$
is asymmetric, because our discussion implies 
that $(L,\{p\}) , (\{p\},\{q\}) , (\{q\},L) \in R_T$.
We may therefore safely assume that there is no concept
of size at least $2$ in $\cL$ that has a non-empty 
intersection with $\cX\sm\{q,q'\}$. Thus, except for the
singletons, the only remaining sets that possibly belong 
to $\cL$ are $\eset$ and $\{q,q'\}$. We still have to show
that not both of them can belong to $\cL$. Assume for the
sake of contradiction that $\eset,\{q,q'\}\in\cL$.
Since $\eset$ is consistent with $T(\{q\})=\{(q',0)\}$,
we have $(\eset,\{q\}) \in R_T$. Clearly, $y_\eset=0$.
Since $\{q\}$ is consistent with every pair $(x,0)$
except for $(q,0)$, we must have $x_\eset = q$.
(Otherwise, we have $(\{q\},\eset) \in R_T$ and arrive
at a contradiction.) Let us now inspect the possible
teaching sets for $L = \{q,q'\}$. Since $\{q,q'\}$ is 
consistent with $T(\{q'\}) = \{(q',1)\}$, 
setting $y_L = 0$ would lead to a contradiction.
The example $(q',1)$ is already in use for 
teaching $\{q'\}$. It is therefore necessary to
set $T(L) = \{(q,1)\}$. An inspection of the various
teaching sets shows that 
$(\eset,\{q\}) , (\{q\},L) , (L,\{q'\}), (\{q'\},\eset) \in R_T$,
which contradicts the fact that $\trcl(R_T)$ is asymmetric.
\item[Case 3:] $\cX^0=\{q,q'\}$ for some $q \neq q' \in \cX$. \\
Note first that $y_{\{q\}}=y_{\{q'\}}=0$ and $y_{\{p\}}=1$ for 
every $p\in\cX\sm\{q,q'\}$. We claim that $\eset\notin\cL$. 
Assume for the sake of contradiction that $\eset\in\cL$. 
Then $(\eset,\{q\}) , (\eset,\{q'\}) \in R_T$ since $\eset$
is consistent with the teaching sets for instances from $\cX^0$.
But then, no matter how $x$ in $T(\eset) = \{(x,0)\}$ is chosen,
at least one of the sets $\{q\}$ and $\{q'\}$ will be consistent
with $T(\eset)$ so that at least one of the pairs $(\{q\},\eset)$
and $(\{q'\},\eset)$ belongs to $R_T$. This contradicts the
fact that $R_T$ must be asymmetric. Thus $\eset\notin\cL$, indeed.
Now it suffices to show that $\cL$ cannot contain a concept
of size at least $2$ that contains an element of $\cX\sm\{q,q'\}$.
Assume for the sake of contradiction that there is a set $L\in\cL$
such that $|L|\ge2$ and $p \in L$ for some $p\in\cX\sm\{q,q'\}$.
Observe that $(L,\{p\}) \in R_T$. Another application of the
first assertion in Lemma~\ref{lem:pbtd1-implications}
shows that $y_L=1$ (because $y_{\{p\}}=1$
and $p \in L$) and $x_L \in \{q,q'\}$ (because the other 
$1$-labeled instances are already in use for teaching 
the corresponding singletons). It follows that one 
of the pairs $(\{q\},L)$ and $(\{q'\},L)$ belongs to $R_T$. 
The third assertion of Lemma~\ref{lem:pbtd1-implications} 
implies that $T(q)=\{(q',0)\}$ or $T(q')=\{(q,0)\}$.
For reasons of symmetry, we may assume that $T(q)=\{(q',0)\}$.
This implies that $(\{p\},\{q\}) \in R_T$.
Let $q''$ be given by $T(q') = \{(q'',0)\}$. Note that 
either $q''=q$ or $q''\in\cX\sm\{q,q'\}$. In the former case,
we have that $(\{p\},\{q'\}) \in R_T$ and in the latter case 
we have that $(\{q\},\{q'\}) \in R_T$. Since $(\{p\},\{q\}) \in R_T$
(which was  observed above already), we conclude that
in both cases, $(\{p\},\{q\}) , (\{p\},\{q'\}) \in \trcl(R_T)$. 
Combining this with our observations above that $(L,\{p\} ) \in R_T$
and that one of the pairs $(\{q\},L)$ and $(\{q'\},L)$ belongs to $R_T$,
yields a contradiction to the fact that $\trcl(R_T)$ is asymmetric.
\end{description}
\end{proof}

\begin{corollary}
Let $\cL\seq2^\cX$ be a concept class that contains $\Sg(\cX)$. 
If $\PBTD(\cL)=1$, then $\RTD(\cL)=1$. 
\end{corollary}

\begin{proof}
According to Theorem~\ref{th:pbtd1}, either $L$ coincides 
with $\Sg(\cX)$ or $\cL$ contains precisely one additional 
concept that is $\eset$ or a set of size $2$.
The partial ordering $\prec$ on $\cL$ that is used in
the first part of the proof of Theorem~\ref{th:pbtd1}
(proof direction ``$\La$'') is easily compiled into
a recursive teaching plan of order $1$ 
for $\cL$.\footnote{This also follows from Lemma~\ref{lem:rtd-pbtd} 
and the fact that there are 
no chains of a length exceeding $2$ in $(\cL,\prec)$.}
\end{proof}


The characterizations proven above can be applied to certain geometric concept classes.
 
Consider a class $\cL$, consisting of bounded and topologically 
closed objects in the $d$-dimensional Euclidean space, that 
satisfies the following condition: for every pair $(A,B)\in\reals^d$, 
there is exactly one object in $\cL$, denoted as $L_{A,B}$ 
in the sequel, such that $A,B \in L$ and such that $\|A-B\|$ 
coincides with the diameter of $L$. This assumption implies 
that $|\cL \sm \Sg(\reals^d)|=\infty$. 
By setting $A=B$, it furthermore implies $\Sg(\reals^d) \seq \cL$. Let us prefer objects with 
a small diameter over objects with a larger diameter. 
Then, obviously, $\{A,B\}$ is a positive teaching 
set for $L_{A,B}$. Because of $|\cL \sm \Sg(\reals^d)|=\infty$, 
$\cL$ does clearly not satisfy the condition 
in Theorem~\ref{th:pbtd1}, which is necessary for $\cL$ to have a PBTD of $1$. We may therefore conclude 
that $\PBTD(\cL) = \PBTD^+(\cL) = 2$.

The family of classes with the required properties is rich 
and includes, for instance, the class of $d$-dimensional balls as well as the class 
of $d$-dimensional axis-parallel rectangles.

\section{Conclusions}

Preference-based teaching uses the natural notion of 
preference relation to extend the classical teaching model. 
The resulting model is (i) more powerful than the classical 
one, (ii) resolves difficulties with the recursive teaching 
model in the case of infinite concept classes, 
and (iii) is at the same time free of coding tricks even 
according to the definition by \cite{GM1996}. Our examples 
of algebraic and geometric concept classes demonstrate that 
preference-based teaching can be achieved very efficiently 
with naturally defined teaching sets and based on intuitive 
preference relations such as inclusion. We believe that 
further studies of the PBTD will provide insights into 
structural properties of concept classes that render them 
easy or hard to learn in a variety of formal learning models.

We have shown that spanning sets lead to a general-purpose 
construction for \linebreak[4]preference-based teaching sets of only 
positive examples. While this result is fairly obvious, 
it provides further justification of the model of 
preference-based teaching, since the teaching sets it 
yields are often intuitively exactly those a teacher 
would choose in the classroom (for instance, one would 
represent convex polygons by their vertices, as in Example~\ref{ex:polygons}). It should 
be noted, too, that it can sometimes be difficult to 
establish whether the upper bound on PBTD obtained this 
way is tight, or whether the use of negative examples 
or preference relations other than inclusion yield 
smaller teaching sets. 
Generally, the choice of preference relation provides 
a degree of freedom that increases the power of the 
teacher but also increases the difficulty of 
establishing lower bounds on the number of 
examples required for teaching.

\medskip
\noindent\textbf{Acknowledgements.} 
Sandra Zilles was supported by the Natural Sciences and Engineering 
Research Council of Canada (NSERC), in the Discovery Grant and Canada 
Research Chairs programs. We thank the anonymous referees for their
numerous thoughtful comments, which greatly helped to improve the
presentation of the paper.
%



\bibliography{PBTDjournal}
\bibliographystyle{plain}

\appendix

\section{Proof of Theorem~\ref{th:bounds-linset}} \label{app:linsets}

In Section~\ref{subsec:shift-lemma}, we present a general 
result which helps to verify the upper bounds in 
Theorem~\ref{th:bounds-linset}. These upper bounds are then
derived in Section~\ref{subsec:linset-ubs}. 
Section~\ref{subsec:linset-lbs} is devoted to the derivation
of the lower bounds.

\subsection{The Shift Lemma} \label{subsec:shift-lemma}

In this section, we assume that $\cL$ is a concept class over 
a universe $\cX \in \{\nats_0,\rats_0^+,\reals_0^+\}$. We furthermore 
assume that $0$ is contained in every concept $L \in \cL$. We can 
extend $\cL$ to a larger class, namely the shift-extension $\cL'$ of $\cL$, 
by allowing each of its concepts to be shifted by some constant 
which is taken from $\cX$:
\[ \cL' = \{c+L:\ (c\in\cX) \wedge (L\in\cL)\} \enspace . \]
The next result states that this extension has little effect only
on the complexity measures $\PBTD$ and $\PBTD^+$:

\begin{lemma} [Shift Lemma]\label{lem:shift}
With the above notation and assumptions, the following holds:
\[ 
\PBTD(\cL) \le \PBTD(\cL') \le 1+\PBTD(\cL)\ \mbox{ and }\ 
\PBTD^+(\cL) \le \PBTD^+(\cL') \le 1+\PBTD^+(\cL) \enspace .
\]
\end{lemma}

\begin{proof}
It suffices to verify the inequalities $\PBTD(\cL') \le 1+\PBTD(\cL)$ 
and $\PBTD^+$ $(\cL') \le 1+\PBTD^+(\cL)$ because the other inequalities 
hold by virtue of monotonicity.
Let $T$ be an admissible mapping for $\cL$. It suffices to show 
that $T$ can be transformed into an admissible mapping $T'$ for $\cL'$
such that $\ord(T') \le 1+\ord(T)$ and such that $T'$ is positive
provided that $T$ is positive. To this end, we define $T'$ as follows:
\[ T'(c+L) = \{(c,+)\} \cup \{(c+x,b):\ (x,b) \in T(L)\} \enspace . \]
Obviously $\ord(T') \le 1+\ord(T)$. Note that $c \in c+L$ because of
our assumption that $0$ is contained in every concept in $\cL$. Moreover, 
since the admissibility 
of $T$ implies that $L$ is consistent with $T(L)$, the above definition 
of $T'(c+L)$ makes sure that $c+L$ is consistent with $T'(c+L)$. It
suffices therefore to show that the relation $\trcl(R_{T'})$ is asymmetric.
Consider a pair $(c'+L',c+L) \in R_{T'}$. By the definition of $R_{T'}$,
it follows that $c'+L'$ is consistent with $T'(c+L)$. 
Because of $(c,+) \in T'(c+L)$, we must have $c' \le c$. Suppose 
that $c'=c$. In this case, $L'$ must be consistent 
with $T(L)$. Thus $L' \prec_T L$. This reasoning implies 
that $(c'+L',c+L) \in R_{T'}$ can happen only if either $c'<c$ 
or $(c'=c) \wedge (L' \prec_T L)$. Since $\prec_T$ is asymmetric,
we may now conclude that $\trcl(R_{T'})$ is asymmetric, as desired.
Finally note that, according to our definition above, the mapping $T'$ 
is positive provided that $T$ is positive. This concludes the proof.
\end{proof}

\subsection{The Upper Bounds in Theorem~\ref{th:bounds-linset}}
\label{subsec:linset-ubs}

We remind the reader that the equality $\PBTD^+(\LINSET_k)=k$ 
was stated in Example~\ref{exmp:pbtdpluslinset}. We will show 
in Lemma~\ref{lem:nelinset-ub} that $\PBTD^+(\NELINSET_k)\le k$.
In combination with the Shift Lemma, this implies
that $\PBTD^+(\LINSET'_k) \le k+1$ 
and $\PBTD^+(\NELINSET'_k) \le k+1$. All remaining upper bounds 
in Theorem~\ref{th:bounds-linset} follow now by virtue of 
monotonicity.

\begin{lemma}\label{lem:nelinset-ub}
$\PBTD^+(\NELINSET_k)\le k$.
\end{lemma}

\begin{proof}
We want to show that there is a preference relation for which $k$ 
positive examples suffice to teach any concept in $\NELINSET_k$.
To this end, let $G=\{g_1,\ldots,g_\ell\}$ be a generator set 
with $\ell \le k$ where $g_1 < \ldots < g_\ell$. 
We use $\sumg(G) = g_1 + \ldots + g_\ell$ to denote the sum of all 
generators in $G$. We say that $g_i$ is a {\em redundant 
generator} in $G$ if $g_i \in \spn{\{g_1,\ldots,g_{i-1}\}}$. 
Let $G^* = \{g_1^*,\ldots,g^*_{\ell^*}\} \seq G$
with $g^*_1 < \ldots < g^*_{\ell^*}$ be the set 
of non-redundant generators in $G$ and
let $\tupleg(G) = (g_1^*,\ldots,g_{\ell^*}^*)$ be 
the corresponding ordered sequence. Then $G^*$ is an independent
subset of $G$ generating the same linear set as $G$ when 
allowing zero coefficients, i.e., we have $\spn{G^*} = \spn{G}$
(although $\spn{G^*}_+ \neq \spn{G}_+$ whenever $G^*$ is 
a proper subset of $G$). 

To define a suitable preference relation, let $G,\widehat G$ be generator 
sets of size $k$ or less with $\tupleg(G) = (g^*_1,\ldots,g^*_{\ell^*})$ 
and $\tupleg(\widehat G) = (\widehat g^*_1,\ldots,\widehat g^*_{\widehat\ell^*})$. 
Let the student prefer $G$ over $\widehat G$ if any of the following 
conditions is satisfied:
\begin{description}
\item[Condition 1:] 
$\sumg(G)>\sumg(\widehat G)$.
\item[Condition 2:]
$\sumg(G)=\sumg(\widehat G)$ and $\tupleg(G)$ is lexicographically
greater than \linebreak[4]$\tupleg$ $(\widehat G)$ without 
having $\tupleg(\widehat G)$ as prefix.
\item[Condition 3:] 
$\sumg(G)=\sumg(\widehat G)$ 
and $\tupleg(G)$ is a proper prefix of $\tupleg(\widehat G)$.
\end{description}

To teach a concept $\spn{G}\in\NELINSET_k$ with $\sumg(G)=g$ 
and $\tupleg(G) = (g_1^*,\ldots,g_{\ell^*}^*)$, one uses the
teaching set 
\[ S = \{ (g,+) , (g+g_1^*,+) ,\ldots, (g+g_{h^*}^*,+) \} \]
where
\begin{equation} \label{eq:def-h}
h = \left\{ \begin{array}{ll}
             \ell^*-1 & \mbox{if $G^*=G$} \\
             \ell^*   & \mbox{if $G^* \subset G$}
    \end{array} \right. \enspace .
\end{equation}
Note that $S$ contains at most $|G| \le k$ examples.
Let $\widehat G$ with $\spn{\widehat G}_+ \in \NELINSET_k$ 
denote the generator set that is returned by the student. 
Clearly $\spn{\widehat G}$ satisfies $\sumg(\widehat G)=g$ 
since 
\begin{itemize}
\item
concepts with larger generator sums are inconsistent 
with $(g,+)$, and 
\item
concepts with smaller generator sums have a lower preference
(compare with Condition~1 above).
\end{itemize}
It follows that $g+g^*_i \in \spn{\widehat G}_+$ is equivalent 
to $g^*_i \in \spn{\widehat G} = \spn{\widehat G^*}$. We conclude 
that the smallest generator in $\tupleg(\widehat G)$ equals $g^*_1$ 
since
\begin{itemize}
\item
a smallest generator in $\tupleg(\widehat G)$ that is greater
than $g^*_1$ would cause an inconsistency with $(g+g^*_1,+)$,
and
\item
a smallest generator in $\tupleg(\widehat G)$ that is smaller
than $g^*_1$ would have a lower preference
(compare with Condition~2 above).
\end{itemize}
Assume inductively that the $i-1$ smallest generators 
in $\tupleg(\widehat G)$ are $g^*_1,\ldots,g^*_{i-1}$. 
Since $g^*_i \notin \spn{\{g^*_1,\ldots,g^*_{i-1}\}}$, we may 
apply a reasoning that is similar to the above reasoning 
concerning $g^*_1$ and conclude that the $i$'th smallest 
generator in $\tupleg(\widehat G)$ equals $g^*_i$.
The punchline of this discussion is that the sequence $\tupleg(\widehat G)$
starts with $g^*_1,\ldots,g^*_{h}$ with $h$ given by~(\ref{eq:def-h}).
Let $G' = G \sm G^*$ be the set of redundant generators in $G$ 
and note that 
\[ 
g - \sum_{i=1}^{h}g^*_i  = \left\{ \begin{array}{ll}
            g^*_{\ell^*}     & \mbox{if $G^* = G$} \\
            \sum_{g' \in G'}g' & \mbox{if $G^* \subset G$} \\
                           \end{array} \right.  \enspace .
\]
Let $\widehat G' = \widehat G \sm \{g^*_1,\ldots,g^*_h\}$.
We proceed by case analysis:
\begin{description}
\item[Case 1:] $G^*=G$. \\
Since $\widehat G$ is consistent with $(g,+)$, 
we have $\sum_{g' \in \widehat G'}g' = g_{\ell^*}^*$.
Since $g^*_{\ell^*} \notin \langle\{g^*_1,\ldots,$ $g^*_{\ell^*-1}\}\rangle$,
the set $\widehat G'$ must contain an element that cannot be 
generated by $g^*_1,\ldots,$ $g^*_{\ell^*-1}$. Given the preferences
of the student (compare with Condition~2), she will 
choose $\widehat G' = \{g^*_{\ell^*}\}$. It follows that $\widehat G = G$.
\item[Case 2:] $G^* \subset G$. \\
Here, we have $\sum_{g' \in \widehat G'}g' = \sum_{g' \in G'}g'$.
Given the preferences of the student (compare with Condition~3), 
she will choose $\widehat G$ such that $\widehat G^* = G^*$ 
and $\widehat G'$ consists of
elements from $\spn{G^*}$ that sum up to $\sum_{g' \in G'}g'$
(with $\widehat G' = \left\{\sum_{g' \in G'}g'\right\}$ among 
the possible choices). Clearly, $\spn{\widehat G}_+ = \spn{G}_+$.
\end{description}
Thus, in both cases, the student comes up with the right
hypothesis.
\end{proof}

\subsection{The Lower Bounds in Theorem~\ref{th:bounds-linset}}
\label{subsec:linset-lbs}

The lower bounds in Theorem~\ref{th:bounds-linset} 
are an immediate consequence of the following result:

\begin{lemma} \label{lem:lbs-linset}
The following lower bounds are valid:
\begin{eqnarray} 
\PBTD^+(\NECFLINSET'_k) & \ge & k+1 \enspace . \label{eq:lb1} \\
\PBTD(\NECFLINSET'_k) & \ge & k-1 \enspace . \label{eq:lb2} \\
\PBTD(\NECFLINSET_k) & \ge & \frac{k-1}{2} \enspace . \label{eq:lb3} \\
\PBTD(\CFLINSET_k) & \ge & k-1 \enspace . \label{eq:lb4} \\
\PBTD^+(\NECFLINSET_k) & \ge & k-1 \enspace . \label{eq:lb5} 
\end{eqnarray}
\end{lemma}

This lemma can be seen as an extension and a strengthening
of a similar result in~\cite{GSZ2015} where the following
lower bounds were shown:
\begin{eqnarray*}
\RTD^+(\NELINSET'_k) & \ge & k+1  \enspace .\\
\RTD(\NELINSET'_k) & \ge & k-1  \enspace .\\
\RTD(\CFLINSET_k) & \ge & k-1 \enspace .
\end{eqnarray*}
The proof of Lemma~\ref{lem:lbs-linset} builds on some
ideas that are found in~\cite{GSZ2015} already, but it
requires some elaboration to obtain the stronger results.

We now briefly explain why the lower bounds
in Theorem~\ref{th:bounds-linset} directly
follow from Lemma~\ref{lem:lbs-linset}.
Note that the lower bound $k-1$ in~(\ref{eq:more-bounds}) 
is immediate from~(\ref{eq:lb2}) and a monotonicity argument. 
This is because $\NELINSET'_k\supseteq\NECFLINSET'_k$ as well as 
$\LINSET'_k\supseteq\CFLINSET'_k\supseteq\NECFLINSET'_k$.
Note furthermore that $\PBTD^+$ $(\CFLINSET'_k) \ge k+1$ 
because of~(\ref{eq:lb1}) and a monotonicity argument.
Then the Shift Lemma implies that $\PBTD^+(\CFLINSET_k) \ge k$.
Similarly, $\PBTD^+(\NELINSET_k$ $)$ $\geq k-1$ follows 
from~(\ref{eq:lb5}) and a monotonicity argument.
All remaining lower bounds in Theorem~\ref{th:bounds-linset}
are obtained from these observations by virtue of monotonicity. 

The proof of Theorem~\ref{th:bounds-linset} can therefore be 
accomplished by proving Lemma~\ref{lem:lbs-linset}. It turns 
out that the proof of this lemma is quite involved. We
will present in Section~\ref{subsubsec:lb0} some
theoretical prerequisites. Sections~\ref{subsubsec:lb1}
and~\ref{subsubsec:lb234} are devoted to the actual
proof of the lemma.

\subsubsection{Some Basic Concepts in the Theory of Numerical Semigroups}
\label{subsubsec:lb0}

Recall from Section~\ref{sec:linsets} that 
$\spn{G} = \left\{\sum_{g \in G}a(g)g: a(g)\in\nats_0\right\}$.
The elements of $G$ are called {\em generators} of $\spn{G}$.
A set $P \subset \nats$ is said to be {\em independent} if 
none of the elements in $P$ can be written as a linear 
combination (with coefficients from $\nats_0$) of the remaining 
elements (so that $\spn{P'}$ is a proper subset of $\spn{P}$
for every proper subset $P'$ of $P$). It is well known~\cite{RG-S2009} 
that independence makes generating systems unique,
i.e., if $P,P'$ are independent, then $\spn{P} = \spn{P'}$ 
implies that $P=P'$. Moreover, for every independent set $P$, 
the following implication is valid:
\begin{equation} \label{eq:independent-set}
(S \seq \spn{P} \wedge P \not\seq S)\ \impl\ (\spn{S} \subset \spn{P}) 
\enspace .
\end{equation}

Let $P = \{a_1,\ldots,a_k\}$ be independent with $a_1 = \min P$. 
It is well known\footnote{E.g., see~\cite{RG-S2009}} and easy 
to see that the residues of $a_1,a_2,\ldots,a_k$ modulo $a_1$ 
must be pairwise distinct (because, otherwise, we would obtain 
a dependence). If $a_1$ is a prime and $|P|\ge2$, then the 
independence of $P$ implies that $\gcd(P)=1$. Thus the following 
holds:

\begin{lemma} \label{lem:ind-gcd}
If $P \subset \nats$ is an independent set of cardinality at least $2$
and $\min P$ is a prime, then $\gcd(P)=1$. 
\end{lemma}

\noindent
In the remainder of the paper, the symbols $P$ and $P'$ 
are reserved for denoting independent sets of generators.

It is well known that $\spn{G}$ is co-finite 
iff $\gcd(G)=1$~\cite{RG-S2009}. Let $P$ be a finite
(independent) subset of $\nats$ such that $\gcd(P)=1$. 
The largest number in $\nats\sm\spn{P}$ is called the 
{\em Frobenius number of $P$} and is denoted as $F(P)$. 
It is well known~\cite{RG-S2009} that 
\begin{equation} \label{eq:frobenius}
F(\{p,q\}) = pq-p-q 
\end{equation}
provided that $p,q \ge 2$ satisfy $\gcd(p,q)=1$.

\subsubsection{Proof of~(\ref{eq:lb1})} \label{subsubsec:lb1}

The shift-extension of $\NECFLINSET_k$ is (by way of definition)
the following class:
\begin{equation} \label{eq:necflinset-extension} 
\NECFLINSET'_k = 
\{c+\spn{P}_+:\ (c\in\nats_0) \wedge (P\subset\nats) \wedge (|P| \le k) \wedge (\gcd(P)=1)\} \enspace . 
\end{equation}
It is easy to see that this can be written alternatively in the form
\begin{equation} \label{eq:necflinset-rewritten} 
\NECFLINSET'_k = 
\left\{N+\spn{P}:\ 
N\in\nats_0 \wedge P\subset\nats \wedge |P| \le k \wedge \gcd(P)=1 \wedge \sum_{p\in P}p \le N\right\} 
\end{equation}
where $N$ in~(\ref{eq:necflinset-rewritten}) corresponds 
to $c+\sum_{p \in P}p$ in~(\ref{eq:necflinset-extension}).

For technical reasons, we define the following subfamilies 
of $\NECFLINSET'_k$. For each $N\ge0$, let
\[ \NECFLINSET'_k[N] = \{N+L : L\in\LINSET_k[N]\} \]
where
\[
\LINSET_k[N] = \left\{\spn{P}\in\LINSET_k : (\gcd(P)=1) \wedge \left(\sum_{p \in P}p \le N\right)\right\}
\enspace .
\]
In other words, $\NECFLINSET'_k[N]$ is the subclass consisting of all
concepts in $\NECFLINSET'_k$ (written in the form~(\ref{eq:necflinset-rewritten}))
whose constant is $N$. 

\noindent
A central notion for proving~(\ref{eq:lb1}) is the following one:
\begin{definition} \label{def:special-set}
Let $k,N\ge2$ be integers. We say that a set $L\in\NECFLINSET'$ 
is {\em $(k,N)$-special} if it is of the form $L = N+\spn{P}$ 
such that the following holds:
\begin{enumerate}
\item
$P$ is an independent set of cardinality $k$ and $\min P$ is a prime
(so that $\gcd(P)=1$ according to Lemma~\ref{lem:ind-gcd}, which 
furthermore implies that $\spn{P}$ is co-finite).
\item 
Let $q(P)$ denote the smallest prime that is greater than $F(P)$
and greater than $\max P$. For $a = \min P$ and $r=0,\ldots,a-1$, 
let
\[ 
t_r(P) = \min \{s\in\spn{P} : s \equiv r \pmod{a}\}\ \mbox{ and }\ 
t_{max}(P) = \max_{0 \le r \le a-1}t_r(P) \enspace . \]
Then 
\begin{equation} \label{eq:large-constant}
N \ge k(a+t_{max}(P))\ \mbox{ and }\  N \ge q(P)+\sum_{p \in P\sm\{a\}}p \enspace .
\end{equation}
\end{enumerate}
\end{definition}

We need at least $k$ positive examples in order to distinguish 
a $(k,N)$-special set from all its proper subsets 
in $\NECFLINSET'_k[N]$, as the following result shows:

\begin{lemma} \label{lem:special-sets}
For all $k\ge2$, the following holds.
If $L\in\NECFLINSET'$ is $(k,N)$-special, then $L \in \NECFLINSET'[N]$
and $I'(L,\NECFLINSET_k[N]) \ge k$.
\end{lemma}

\begin{proof}
Suppose that $L = N+\spn{P}$ is of the form as described in 
Definition~\ref{def:special-set}. Let $P = \{a,a_2\ldots,a_k\}$ 
with $a = \min P$. For the sake of simplicity, we will write $t_r$ 
instead of $t_r(P)$ and $t_{max}$ instead of $t_{max}(P)$. 
The independence of $P$ implies that $t_{a_i \bmod a} = a_i$ 
for $i=2,\ldots,k$. It follows that $t_{max} \ge \max P$. Since, 
by assumption, $N \ge k \cdot t_{max}$, it becomes obvious 
that $L \in \NECFLINSET'[N]$. 

Assume by way of contradiction that the following holds:
\begin{itemize}
\item[(A)]
There is a weak spanning set $S$ of size $k-1$ for $L$ 
w.r.t.~$\NECFLINSET'_k[N]$. 
\end{itemize}
Since $N$ is contained in any concept from $\NECFLINSET'_k[N]$, we may
assume that $N \notin S$ so that $S$ is of the 
form $S = \{N+x_1,\ldots,N+x_{k-1}\}$ for integers $x_i\ge1$. 
For $i = 1,\ldots,k-1$, let $r_i = x_i \bmod{a} \in \{0,1,\ldots,a-1\}$. 
It follows that each $x_i$ is of the form $x_i = q_ia+t_{r_i}$ for some 
integer $q_i\ge0$. Let $X = \{x_1,\ldots,x_{k-1}\}$. We proceed by case 
analysis:
\begin{description}
\item[Case 1:] $X \seq \{a_2,\ldots,a_k\}$ 
(so that, in view of $|X|=k-1$, we even have $X = \{a_2,\ldots,a_k\}$). \\ 
Let $L' = N+\spn{X}$. Then $S \seq L'$. Note that $X \seq P$ but $P \not\seq X$.
We may conclude from~(\ref{eq:independent-set}) that $\spn{X} \subset \spn{P}$ 
and, therefore, $L' \subset L$. Thus $L'$ is a proper subset of $L$ which
contains $S$. Note that~(\ref{eq:large-constant}) implies 
that $N \ge \sum_{i=2}^{k}a_i = \sum_{i=1}^{k-1}x_i$. 
If $\gcd(X)=1$, then $L' \in \NECFLINSET[N]$ and we have an
immediate contradiction to the above assumption~(A). Otherwise, 
if $\gcd(X)\ge2$, then we define $L'' = N+\spn{X\cup\{q(P)\}}$. 
Note that $S \seq L' \seq L''$. Since $q(P)>F(P)$, we 
have $X\cup\{q(P)\} \seq \spn{P}$ and, since $q(P) > \max P$, we 
have $P \not\seq X\cup\{q(P)\}$. We may conclude 
from~(\ref{eq:independent-set}) that $\spn{X\cup\{q(P)\}} \subset \spn{P}$ 
and, therefore, $L'' \subset L$. Thus, $L''$ is a proper subset of $L$
which contains $S$.  Because $X = \{a_2,\ldots,a_k\}$ and $q(P)$ is a prime 
that is greater than $\max P$, it follows that $\gcd(X\cup\{q(P)\}) = 1$.
In combination with~(\ref{eq:large-constant}), it easily follows now
that $L'' \in \NECFLINSET[N]$. Putting everything together, we arrive at
a contradiction to the assumption~(A).
\item[Case 2:] $X \not\seq \{a_2,\ldots,a_k\}$. \\
If $r_i=0$ for $i=1,\ldots,k-1$, then each $x_i$ is a multiple of $a$.
In this case, $N+\spn{a,q(P)}$ is a proper subset of $L = N+\spn{P}$ 
that is consistent with $S$, which yields a contradiction.
We may therefore assume that there exists $i' \in\{1,\ldots,k-1\}$ such
that $r_{i'}\neq0$. From the case assumption, $X \not\seq \{a_2,\ldots,a_k\}$,
it follows that there must exist an index $i''\in\{1,\ldots,k-1\}$ such 
that $q_{i''} \ge 1$ or $t_{r_{i''}}\notin\{a_2,\ldots,a_k\}$. For $i=1,\ldots,k-1$, 
let $q'_i = \min\{q_i,1\}$ and $x'_i = q'_ia+t_{r_i}$. Note that $q'_{i''}=1$
iff $q_{i''} \ge 1$. Define $L'' = N+\spn{X'}$ for $X' = \{a,x'_1,\ldots,x'_{k-1}\}$ 
and observe the following. First, the set $L''$ clearly contains $S$. Second, 
the choice of $x'_1,\ldots,x'_{k-1}$ implies that $X' \seq \spn{P}$.
Third, it easily follows from $q'_{i''}=1$ or $t_{r_{i''}}\notin\{a_2,\ldots,a_k\}$
that $P \not\seq \{a,x'_1,\ldots,x'_{k-1}\}$. We may conclude
from~(\ref{eq:independent-set}) that $\spn{X'} \subset \spn{P}$
and, therefore, $L'' \subset L$. Thus, $L''$ is a proper subset of $L$
which contains $S$. Since $r_{i'} \neq 0$ and $a$ is a prime, it follows 
that $\gcd(a,x'_{i'})=1$ and, therefore, $\gcd(X')=1$. In combination 
with~(\ref{eq:large-constant}), it easily follows now 
that $L'' \in \NECFLINSET[N]$. Putting everything together, we obtain 
again  a contradiction to the assumption~(A).
\end{description}
\end{proof}

For the sake of brevity, let $\cL = \NECFLINSET'$. 
Assume by way of contradiction that there exists a positive mapping $T$ 
of order $k$ that is admissible for $\cL_k$. We will pursue the following strategy:
\begin{enumerate}
\item 
We define a set $L \in \cL_k$ of the form $L = N+p+\spn{1}$.
\item
We define a second set $L' = N+\spn{G} \in \cL$ that is 
$(k,N)$-special and consistent with $T^+(L)$. Moreover, $L' \sm L = \{N\}$.
\end{enumerate}
If this can be achieved, then the proof will be accomplished as follows:
\begin{itemize}
\item
According to Lemma~\ref{lem:special-sets}, $T^+(L')$ must contain
at least $k$ examples (all of which are different from $N$) for 
distinguishing $L'$ from all its proper subsets in $\cL_k[N]$.
\item
Since $L'$ is consistent with $T^+(L)$, the set $T^+(L')$ must contain
an example which distinguishes $L'$ from $L$. But the only example 
which fits this purpose is $(N,+)$.
\item
The discussion shows that $T^+(L')$ must contain $k$ examples in order
to distinguish $L'$ from all its proper subsets in $\cL_k$ plus
one additional example, $N$, needed to distinguish $L'$ from $L$.
\item
We obtain a contradiction to our initial assumption that $T^+$
is of order $k$.
\end{itemize}
We still have to describe how our proof strategy can actually be implemented. 
We start with the definition of $L$. Pick the smallest prime $p \ge k+1$. Then
$\{p,p+1,\ldots,p+k\}$ is independent. 
Let $M = F(\{p,p+1\}) \stackrel{(\ref{eq:frobenius})}{=} p(p+1)-p-(p+1)$. 
An easy calculation shows that $k\ge2$ and $p \ge k+1$ imply that $M \ge p+k$. 
Let $I = \{p,p+1,\ldots,M\}$.  Choose $N$ large enough so that all concepts 
of the form 
\[ N+\spn{P}\ \mbox{ where}\ |P|=k,\ p = \min P \mbox{ and } P \seq I \]
are $(k,N)$-special. With these choices of $p$ and $N$, let $L = N+p+\spn{1}$.
Note that $N+p,N+p+1 \in T^+(L)$ because, otherwise, one 
of the concepts $N+p+1+\spn{1},N+p+\spn{2,3} \subset L$ would be consistent 
with $T^+(L)$ whereas $T^+(L)$ must distinguish $L$ from all its proper subsets in
$\cL_k$. Setting $A = \{x: N+x \in T^+(L)\}$, it follows that $|A| = |T^+(L)| \le k$ 
and $p,p+1 \in A$. The set $A$ is not necessarily independent but it 
contains an independent subset $B$ such that $p,p+1 \in B$ 
and $\spn{A} = \spn{B}$. Since $M = F(\{p,p+1\})$, it follows that any integer 
greater than $M$ is contained in $\spn{p,p+1}$. Since $B$ is an independent 
extension of $\{p,p+1\}$, it cannot contain any integer greater than $M$. 
It follows that $B \seq I$. Clearly, $|B| \le k$ and $\gcd(B)=1$. We would 
like to transform $B$ into another generating system $G \seq I$ such that
\[ \spn{B} \seq \spn{G}, \gcd(G) = 1 \mbox{ and } |G|=k \enspace . \]
If $|B|=k$, we can simply set $G=B$. If $|B|<k$, then we make use of
the elements in the independent set $\{p,p+1,\ldots,p+k\} \seq I$ 
and add them, one after the other, to $B$ (thereby removing other elements 
from $B$ whenever their removal leaves $\spn{B}$ invariant) until the 
resulting set $G$ contains $k$ elements. We now define the set $L'$ by 
setting $L' = N+\spn{G}$. Since $G \seq I = \{p,p+1,\ldots,M\}$, and
$p,p+1 \in G$, it follows that $p = \min G$, $\gcd(G)=1$ and
$\min(L' \sm\{N\})$ is $N+p$. Thus, $L' \sm L = \{N\}$, 
as desired. Moreover, since $N$ had been chosen large enough, the set $L'$
is $(k,N)$-special. Thus $L$ and $L'$ have all properties that 
are required by our proof strategy and the proof of~(\ref{eq:lb1})  
is complete.

\subsubsection{Proof of~(\ref{eq:lb2}),~(\ref{eq:lb3}),~(\ref{eq:lb4}) and~(\ref{eq:lb5})}
\label{subsubsec:lb234}

\noindent
We make use of some well known (and trivial) lower bounds on $\TD_{min}$:

\begin{example} \label{ex:powerset}
For every $k\in\nats$, let $[k] = \{1,2,\ldots,k\}$, let $2^{[k]}$ denote
the powerset of $[k]$ and, for all $\ell = 0,1,\ldots,k$, let
\[ {[k] \choose \ell} = \{S \seq [k]:\ |S|=\ell\} \]
denote the class of those subsets of $[k]$ that have exactly $\ell$ elements.
It is trivial to verify that
\[ 
\TD_{min}\left(2^{[k]}\right) = k\ \mbox{ and }\ 
\TD_{min}\left({[k] \choose \ell}\right) = \min\{\ell,k-\ell\} \enspace .
\]
\end{example}

In view of $\PBTD^+(\LINSET_k) = k$, the next results
show that negative examples are of limited help only as far as
preference-based teaching of concepts from  $\LINSET_k$ is
concerned:

\begin{lemma} \label{lem:td-min}
For every $k\ge1$ and for all $\ell=0,\ldots,k-1$, let
\begin{eqnarray*}
\cL_k & = & \{\spn{k,p_1,\ldots,p_{k-1}}:\ p_i \in \{k+i,2k+i\}\} \enspace , \\
\cL_{k,\ell} & = &\{ \{\spn{k,p_1,\ldots,p_{k-1}}\in\cL_k:\ |\{i : p_i=k+i\}| = \ell\} 
\enspace .
\end{eqnarray*}
With this notation, the following holds:
\[ 
\TD_{min}(\cL_k) \ge k-1\ \mbox{ and }\ 
\TD_{min}(\cL_{k,\ell}) \ge \min\{\ell,k-1-\ell\} \enspace .
\]
\end{lemma}

\begin{proof}
For $k=1$, the assertion in the lemma is vacuous. Suppose therefore
that $k\ge2$. An inspection of the generators $k,p_1,\ldots,p_{k-1}$
with $p_i \in \{k+i,2k+i\}$ shows that 
\begin{eqnarray*} 
\cL_k & = & \{L_{k,S}:\ S \seq \{k+1,k+2,\ldots,2k-1\}\} \\ 
\cL_{k,\ell} & = & \{L_{k,S}:\ (S \seq \{k+1,k+2,\ldots,2k-1\}) \wedge (|S|=\ell)\}\ 
\end{eqnarray*}
where
\[ L_{k,S} = \{0,k\} \cup \{2k,2k+1,\ldots\} \cup S \enspace . \]
Note that the examples in $\{0,1,\ldots,k\} \cup \{2k,2k+1,\ldots,\}$
are redundant because they do not distinguish between distinct concepts
from $\cL_k$. The only useful examples are therefore contained in 
the interval $\{k+1,k+2,\ldots,2k-1\}$. From this discussion, it follows
that teaching the concepts of $\cL_k$ (resp.~of $\cL_{k,\ell}$) is not
essentially different from teaching the concepts of $2^{[k-1]}$ 
$\left(\mbox{resp.~of }{[k-1] \choose \ell}\right)$. This completes 
the proof of the lemma because we know from Example~\ref{ex:powerset} 
that $\TD_{min}(2^{[k-1]}) = k-1$ 
and $\TD_{min}\left({[k-1] \choose \ell}\right) = \min\{\ell,k-1-\ell\}$.
\end{proof}


We claim now that the 
inequalities~(\ref{eq:lb2}),~(\ref{eq:lb3}) and~(\ref{eq:lb4})
are valid, i.e., we claim that the following holds:
\begin{enumerate}
\item
$\PBTD(\CFLINSET_k) \ge k-1$.
\item
$\PBTD(\NECFLINSET_k) \ge \lfloor (k-1)/2 \rfloor$.
\item
$\PBTD(\NECFLINSET'_k) \ge k-1$.
\end{enumerate}

\begin{proof}
For $k=1$, the inequalities are obviously valid. Suppose therefore
that $k\ge2$.
\begin{enumerate}
\item
Since $\gcd(k,k+1) = \gcd(k,2k+1) = 1$, it follows that $\cL_k$ is 
a finite subclass of $\CFLINSET_k$. 
Thus $\PBTD(\CFLINSET_k) \ge \PBTD(\cL_k) \ge \TD_{min}(\cL_k) \ge k-1$.
\item
Define $\cL_{k}[N] = \{N+L:\ L\in\cL_k\}$ 
and $\cL_{k,\ell}[N] = \{N+L:\ L\in \cL_{k,\ell}\}$. 
Clearly $\TD_{min}(\cL_k[N]) = \TD_{min}(\cL_k)$ 
and $\TD_{min}(\cL_{k,\ell}[N]) = \TD_{min}(\cL_{k,\ell})$ holds
for every $N\ge0$. It follows that the lower bounds in Lemma~\ref{lem:td-min}
are also valid for the classes $\cL_k[N]$ and $\cL_{k,\ell}[N]$ in place
of $\cL_k$ and $\cL_{k,\ell}$, respectively. Let 
\begin{equation} \label{eq:special-shift} 
N(k) = k^2 + (k-1-\lfloor (k-1)/2 \rfloor)k + \sum_{i=1}^{k-1}i = 
k^2+(k-1-\lfloor (k-1)/2 \rfloor)k+\frac{1}{2}(k-1)k \enspace .
\end {equation}
It suffices to show that $N(k)+\cL_{k,\lfloor (k-1)/2 \rfloor}$ is a finite
subclass of $\NECFLINSET_k$. To this end, first note that
\[ 
\spn{k,p_1,\ldots,p_{k-1}}_+ = k+\sum_{i=1}^{k-1}p_i + \spn{k,p_1,\ldots,p_{k-1}}
\enspace .
\]
Call $p_i$ ``light'' if $p_i = k+i$ and call it ``heavy'' if $p_i = 2k+i$.
Note that a concept $L$ from $N(k)+\cL_{k,\ell}$ is of the general form
\begin{equation} \label{eq:general-form}
L = N(k) + \spn{k,p_1,\ldots,p_{k-1}} 
\end{equation}
with exactly $\ell$ light parameters among $p_1,\ldots,p_{k-1}$. 
A straightforward calculation shows that, for $\ell = \lfloor (k-1)/2 \rfloor$,
the sum $k+\sum_{i=1}^{k-1}p_i$ equals the number $N(k)$ as defined
in~(\ref{eq:special-shift}). Thus, the concept $L$ from~(\ref{eq:general-form}) 
with exactly $\lfloor (k-1)/2 \rfloor$ light parameters among $\{p_1,\ldots,p_{k-1}\}$ 
can be rewritten as follows:
\[ 
L = N(k) + \spn{k,p_1,\ldots,p_{k-1}} = \spn{k,p_1,\ldots,p_{k-1}}_+
\enspace .
\]
This shows that $L \in \NECFLINSET_k$. As $L$ is a concept 
from $N(k)+$ \linebreak[4]$\cL_{k,\lfloor (k-1)/2 \rfloor}$ in general form, we may conclude
that $N(k)+\cL_{k,\lfloor (k-1)/2 \rfloor}$ is a finite subclass of $\NECFLINSET_k$,
as desired.
\item
The proof of the third inequality is similar to the above proof 
of the second one. It suffices to show that, for every $k\ge2$, 
there exists $N \in \nats$ such that $N+\cL_k$ is a subclass
of $\NECFLINSET'_k$. To this end, we set $N = 3k^2$. A concept $L$
from $3k^2+\cL_k$ is of the general form
\[ L = 3k^2+\spn{k,p_1,\ldots,p_{k-1}} \]
with $p_i \in \{k+i,2k+i\}$ (but without control over the number of light
parameters). It is easy to see that the constant $3k^2$ is large enough
so that $L$ can be rewritten as
\[ L = 3k^2 - \left(k+\sum_{i=1}^{k-1}p_i\right) + \spn{k,p_1,\ldots,p_{k-1}}_+ \]
where $3k^2 - \left(k+\sum_{i=1}^{k-1}p_i\right) \ge 0$.
This shows that $L \in \NECFLINSET'_k$. As $L$ is a concept
from $3k^2+\cL_k$ in general form, we may conclude 
that $3k^2+\cL_k$ is a finite subclass of $\NECFLINSET'_k$, as desired.

\end{enumerate}
\end{proof}

We conclude with the proof of the inequality~(\ref{eq:lb5}).

\begin{lemma} \label{lem:pbtdplusnecflinsetklb}
$\PBTD^+(\NELINSET_k) \geq \PBTD^+(\NECFLINSET_k) \geq k-1$.
\end{lemma}

\begin{proof}
The class $\NECFLINSET_1$ contains only $\natnum$, 
and so $\PBTD^+(\NECFLINSET_1)$ $= 0$.
The class $\NECFLINSET_2$ contains at least two members 
so that $\PBTD^+$ \linebreak[4]$(\NECFLINSET_2)$ $\geq 1$.
Now assume $k \geq 3$. Set 
\[ N = \sum_{i=0}^{k-1} \left(k+i\right) \]
and 
\[ 
L = \spn{k,k+1,\ldots,2k-1}_+ = N+\spn{k,k+1,\ldots,2k-1} 
  = \{N\} \cup \{N+k,N+k+1,\ldots\} \enspace .
\]
Choose and fix an arbitrary set $S \seq L$ of size $k-2$.
It suffices to show $S$ is not a weak spanning set for $L$ 
w.r.t.~$\NECFLINSET_k$. If $S$ does not contain $N$, then 
the set 
\[ L' = \spn{N+k-1,1}_+ = L\sm\{N\} \]
satisfies $S \subset L' \subset L$ so that $S$ cannot be 
a weak spanning set for $L$. 
Suppose therefore from now on that $N \in S$.
We proceed by case analysis:
\begin{description}
\item[Case 1:] $k = 3$. \\
Then $N = 12$, $L = 12+\spn{3,4,5} =  \{12\} \cup \{15,16,17,\ldots\}$.
Moreover $|S|=1$ so that $S=\{12\}$. 
Now the set $L' = \spn{5,7}_+ = 12+\spn{5,7}$ 
satisfies $S \subset L' \subset L$ so that $S$ cannot be
a weak spanning set for $L$.
\item[Case 2:] $k = 4$. \\
Then $N = 22$, $L = 22+\spn{4,5,6,7} =  \{22\} \cup \{26,27,28,\ldots\}$.
Moreover $|S|=2$ so that $S=\{22\} \cup \{26+x\}$ for some $x\ge0$.
Let $a = (x \bmod 4) \in \{0,1,2,3\}$. It is easy to check that
the set
\[ 
L' = \left\{ \begin{array}{ll}
       22+\spn{4,5,13} & \mbox{if $a\in\{0,1\}$} \\
       22+\spn{4,7,11} & \mbox{if $a=3$} \\
       22+\spn{5,6,11} & \mbox{if $x=a=2$} \\
       22+\spn{4,5,13} & \mbox{if $x>a=2$}
     \end{array} \right . 
\]
satisfies $S \subset L' \subset L$ so that $S$ cannot be 
a weak spanning set for $L$.
\item[Case 3:] $k \geq 5$. \\
Then the set $S$ has the form 
$S = \{N\} \cup \{N+k+x_1,\ldots,N+k+x_{k-3}\}$
for distinct integers $x_1,\ldots,x_{k-3}\ge0$. For $i=1,\ldots,k-3$,
let $a_i = (x_i \bmod k) \in \{0,\ldots,k-1\}$. The the set
\[ L' = N + \spn{k,k+a_1,\ldots,k+a_{k-3},N-(k-2)k-(a_1+\ldots+a_{k-3})} \]
satisfies $S \subset L' \subset L$ so that $S$ cannot be 
a weak spanning set for $L$.
\end{description}
In any case, we came to the conclusion that a subset of $L$
with only $k-2$ elements cannot be a weak spanning set for $L$
w.r.t.~$\NECFLINSET_k$.
\end{proof}

%

\end{document}